\def\isarxiv{1} 

\ifdefined\isarxiv
\documentclass[11pt]{article}

\usepackage[numbers]{natbib}

\else

\documentclass[twoside]{article}

\usepackage{aistats2025}

\fi

\usepackage{amsmath}
\usepackage{amsthm}
\usepackage{amssymb}
\usepackage{algorithm}
\usepackage{algpseudocode}
\usepackage{grffile}
\usepackage{wrapfig,epsfig}
\usepackage{url}
\usepackage{xcolor}
\usepackage{epstopdf}

\usepackage{bbm}
\usepackage{dsfont}

\allowdisplaybreaks

\ifdefined\isarxiv

\usepackage{tikz}
\usepackage{hyperref}  
\hypersetup{colorlinks=true,citecolor=blue,linkcolor=blue} 
\usetikzlibrary{arrows}
\usepackage[margin=1in]{geometry}

\else

\usepackage{hyperref}

\fi

\theoremstyle{plain}
\newtheorem{theorem}{Theorem}[section]
\newtheorem{lemma}[theorem]{Lemma}
\newtheorem{definition}[theorem]{Definition}

\newtheorem{fact}[theorem]{Fact}

\newcommand{\wh}{\widehat}
\newcommand{\wt}{\widetilde}

\newcommand{\R}{\mathbb{R}}

\renewcommand{\hat}{\wh}

\DeclareMathOperator*{\E}{{\mathbb{E}}}

\DeclareMathOperator{\D}{\mathsf{D}}

\DeclareMathOperator{\diag}{diag}

\makeatletter
\newcommand*{\RN}[1]{\expandafter\@slowromancap\romannumeral #1@}
\makeatother

\usepackage{lineno}

\begin{document}

\ifdefined\isarxiv

\date{}

\title{Differentially Private Attention Computation}
\author{
Yeqi Gao\thanks{\texttt{ a916755226@gmail.com}. The University of Washington.}
\and
Zhao Song\thanks{\texttt{ magic.linuxkde@gmail.com}. The Simons Institute for the Theory of Computing at the University of California, Berkeley.}
\and
Xin Yang\thanks{\texttt{ yangxin199207@gmail.com}. The University of Washington.}
\and 
Yufa Zhou\thanks{\texttt{ yufazhou@seas.upenn.edu}. University of Pennsylvania.}
}

\else

\twocolumn[

\aistatstitle{
Differentially Private Attention Computation
}

\aistatsauthor{ Author 1 \And Author 2 \And  Author 3 }

\aistatsaddress{ Institution 1 \And  Institution 2 \And Institution 3 } ]

\fi

\ifdefined\isarxiv
\begin{titlepage}
  \maketitle
  \begin{abstract}
Large language models (LLMs), especially those based on the Transformer architecture, have had a profound impact on various aspects of daily life, such as natural language processing, content generation, research methodologies, and more.
Nevertheless, a crucial concern regarding the inference results of large language models is the issue of security and privacy.
Given that large language models can generate results that may leak sensitive confidential or copyright information in many scenarios, it is crucial to compute the attention matrix with provable privacy guarantees, as attention is all you need.

In this work, we propose a novel and efficient algorithm for approximating the attention matrix while providing differential privacy (DP) guarantees. To achieve this, we build on recent advancements in fast attention computation and differentially private matrix publishing.

  \end{abstract}
  \thispagestyle{empty}
\end{titlepage}

{\hypersetup{linkcolor=black}
\tableofcontents
}
\newpage

\else
\begin{abstract}
Large language models (LLMs), especially those based on the Transformer architecture, have had a profound impact on various aspects of daily life, such as natural language processing, content generation, research methodologies, and more.
Nevertheless, a crucial concern regarding the inference results of large language models is the issue of security and privacy.
Given that large language models can generate results that may leak sensitive confidential or copyright information in many scenarios, it is crucial to compute the attention matrix with provable privacy guarantees, as attention is all you need.

In this work, we propose a novel and efficient algorithm for approximating the attention matrix while providing differential privacy (DP) guarantees. To achieve this, we build on recent advancements in fast attention computation and differentially private matrix publishing.

\end{abstract}

\fi

\section{INTRODUCTION}

The development of large language models (LLMs) has been rapid and significant in recent years, with numerous breakthroughs and advancements in the field. BERT \cite{dclt18} achieved state-of-the-art performance on a wide range of language tasks by training on a massive amount of text data in 2018. Since then, the GPT (Generative Pre-trained Transformer) family of models has further advanced the field. GPT-2 \cite{rwc+19} and GPT-3 \cite{bmr+20}, with billions of parameters, are able to generate highly coherent and human-like text. Other notable LLMs include XLNet \cite{ydy+19}, which addresses some of the limitations of BERT \cite{dclt18}, and RoBERTa \cite{log+19}, which improves upon BERT \cite{dclt18}'s training methods to achieve better performance. The rapid development of LLMs has been fueled by advancements in hardware, software, and data availability, allowing researchers and companies to train and deploy these models at an unprecedented scale.

As a result of their development, LLMs have found a wide range of applications in various fields. In the field of natural language processing (NLP) \cite{vsp+17,rns+18,dclt18,bmr+20}, LLMs are used for tasks such as language translation \cite{hwl21}, sentiment analysis \cite{uas+20}, and creative writing \cite{o23}. In addition, LLMs are being used to develop chatbots and virtual assistants that can understand and respond to natural language queries \cite{bmr+20,o23}. Outside of NLP, LLMs are being used in scientific research to generate new hypotheses and discover novel patterns in large datasets. The applications of LLMs are expanding rapidly, and it is likely that they will play an increasingly important role in many fields, such as computer vision \cite{rf18}, robotics \cite{knk21}, and autonomous vehicles \cite{ztl+17,bko18}.

Despite their many benefits, large language models (LLMs) have the potential to pose several privacy and security risks \cite{s18,vkb23,kgw+23,emm+23}. One concern is the risk of data breaches, as LLMs require large amounts of data to be trained and the data used for training is often collected from public sources without the explicit consent of the individuals involved. This data could include sensitive personal information, such as medical records, financial data, or personally identifiable information \cite{tpg17,erld17}. Furthermore, LLMs can potentially be used to generate convincing fake text \cite{rwc+19,rsr+20}, which could be used for malicious purposes such as phishing attacks, spreading misinformation or impersonating individuals online. Additionally, LLMs can be used for so-called "model inversion" attacks \cite{fjr15}, where an attacker can extract private information about individuals by querying the model. For example, an attacker could use the model to infer sensitive information, such as an individual's political views or sexual orientation, based on their text input. These privacy and security concerns highlight the need for ethical considerations and responsible use of LLMs, as well as for the development of robust security mechanisms to protect against potential attacks.

Given that attention mechanisms are at the core of models like the Transformer \cite{vsp+17}, considering privacy in attention computation is crucial. Attention mechanisms process input data that may contain sensitive information, and the computed attention weights could inadvertently reveal this information if exposed. Specifically, if sensitive data are encoded within the attention weights, compromising these weights could lead to the disclosure of personal identifying information or trade secrets.

Recent studies have focused on privacy issues related to Transformers and their attention mechanisms. For example, \cite{vkb23} showed that learned conditional generative models might output samples similar to copyrighted data in their training set, leading to copyright infringement issues. The proposed solution is near access-freeness (NAF), which involves defining generative models that do not access potentially copyrighted data. \cite{vkb23} provide formal definitions of NAF and generative model learning algorithms that produce models with strong bounds on the probability of sampling protected content. While NAF provides formal guarantees against such infringements, it also underscores the need to secure the attention mechanisms within Transformers to prevent privacy breaches related to sensitive information embedded in attention weights.

Moreover, the potential harms of LLMs extend to intellectual property violations and the dissemination of misinformation. To mitigate these issues, \cite{kgw+23} developed a watermarking framework for proprietary language models. This framework embeds invisible signals into generated text that can be algorithmically detected, promoting accountability and traceability. Although this approach addresses the outputs of LLMs, it emphasizes the broader necessity of safeguarding the internal computations—like attention mechanisms—to enhance overall security.

Building upon the discussed challenges, our research focuses on addressing privacy and security issues in attention computation. Unlike previous works \cite{zhdk23, as23, bsz23, lsz23, dls23, gms23, dms23, gsy23}, our work will concentrate on static computation for attention computation. To be specific, static computation is a technique used in implementing attention mechanisms in deep learning models, especially in the field of natural language processing. It involves computing the attention weights between the encoder and decoder only once and reusing them during decoding, rather than dynamically computing the attention weights for each time step during decoding. This approach enhances computational efficiency and reduces overall decoding time, especially for longer sequences, while also strengthening privacy and security in attention-based models.

\subsection{Key Definitions}
Here, let us recall the formal mathematical definition of attention computation in static setting,
\begin{definition}[Attention computation, see \cite{zhdk23,as23,bsz23} as examples]
Given matrices $Q \in \R^{n \times d}$, $K \in \R^{n \times d}$ and $V \in \R^{n \times d}$, the goal is to compute
\begin{align*}
\mathsf{Att}(Q,K,V) := D^{-1} A V
\end{align*}
where $A = \exp(Q K^\top) \in \R^{n \times n}$ (we apply $\exp()$ entry-wisely to the matrix), and $D = \diag( A {\bf 1}_n )$.
\end{definition}

Following from the setting of work \cite{dms23}, we consider the symmetric attention approximation problem where we treat $Q= K$ and ignore the effect of $V$. The formal formulation is
\begin{definition}
Given $X \in \R^{n \times d}$, the goal is to find some $Y \in \R^{n \times m}$ such that
\begin{align*}
    \| D(X X^\top)^{-1} \exp(XX^\top) - D(Y Y^\top)^{-1} \exp(YY^\top) \| \leq \mathrm{small}
\end{align*}
where $\| \cdot \|$ is some certain norm and $D(X X^\top ) = \diag( \exp(XX^\top) \cdot {\bf 1}_n )$.
\end{definition}

One recent work \cite{vkb23} choose the angle of near access-freeness to study the privacy concerns in LLMs. However, in this work, we use the differential privacy concept \cite{dr14}, and the formal definition of differential privacy can be written as follows.

\begin{definition}[Differential Privacy~\cite{dmns06,dkm+06}]
\label{def:dp}
A randomized mechanism ${\cal M}$ is $(\epsilon,\delta)$-differentially private if for any event ${\cal O}\in \mathrm{Range}({\cal M})$ and for any pair of neighboring databases $S, S'$ that differ in a single data element, one has
\begin{align*}
    \Pr[{\cal M}(S)\in {\cal O}] \leq \exp(\epsilon)\cdot \Pr[{\cal M}(S')\in {\cal O}]+\delta.
\end{align*}
\end{definition}

Finally, we're ready to define our differentially private attention computation problem.

\begin{definition}[General Differentially Private Attention]\label{def:general_dp_attention}
Let $f: \R \rightarrow \R$ denote some fixed function. For a given matrix $X \in \R^{n \times d}$ with $d \gg n$, let ${\cal M}$ denote some mapping that maps $\R^{n \times d}$ to $\R^{n \times n}$, let $A = {\cal M}(X)$, for parameter $\epsilon, \delta \in (0,0.1)$, the goal is to design an $(\epsilon,\delta)$-differetially private algorithm that takes $X \in \R^{n \times d}$ as input and generates a PSD matrix $B \in \R^{n \times n}$ such that
\begin{align*}
\| \D(A)^{-1} f(A) - \D(B)^{-1} f(B) \| \leq g(\epsilon, \delta)
\end{align*}
where $f(A)_{i,j} = f(A_{i,j})$, $\D(A) = \diag( f(A) {\bf 1}_n )$ and where $g$ is some function.
\end{definition} 
Definition \ref{def:general_dp_attention} is very general, and covers the standard self-attention computation.
In particular, when ${\cal M}(X) = XX^\top$ and $f(z) = \exp(z)$, then above definition recovers the standard self-attention in LLMs.

\subsection{Our Result}
\label{sub:intro:result}

Our results rely on good properties of the input data, which are defined as follows. They play a crucial role in the analysis of sensitivity with respect to ${\cal M}(X) = X X^\top$ (See Section~\ref{sec:sensitivity}).
\begin{definition}[Dataset]\label{def:dataset}
Fix $\eta > 0 , \alpha>0$. We say our dataset $X \in \R^{n \times d}$ is $(\alpha, \eta)$-good if
$XX^\top \succeq \eta \cdot I_n$ and for all $i \in [d]$, $\| X_{*,i}  \|_2 \leq \alpha$.
\end{definition}
In addition, we will introduce our proposed definition of neighboring data as follows.
\begin{definition}[Neighboring data]\label{def:neighboring}
Let $X, \wt{X} \in \R^{n \times d}$ denote two datasets from distribution ${\cal D}$, we say that $X$ and $\wt{X}$ are $\beta$-close if there exists exact one $i \in [d]$ so that $\| X_{*,i} - \wt{X}_{*,i} \|_2 \leq \beta$ and for all $j \in [d] \backslash \{i\}$, $X_{*,j} = \wt{X}_{*,j}$.
In this work, we consider two datasets to be neighboring if they are $\beta$-close.
\end{definition}
The above definition facilitates a more straightforward analysis of the sensitivity of attention matrix computations. By regulating $\beta$-closeness, we can establish bounds on how the attention matrix responds to minor variations in input data, which is essential for ensuring differential privacy guarantees. Furthermore, in practical scenarios, assessing dataset similarity based on feature-wise differences rather than individual data points can be more practical and aligns better with real-world considerations.

Based on the aforementioned definitions, our work demonstrates the sensitivity property of ${\cal M}(X) = X X^\top$ (attention matrix computation). Furthermore, we present a novel and efficient algorithm for approximating the attention matrix, which combines error analysis on matrix perturbation with provable privacy guarantees. We state our result as follows:
\begin{theorem}[Main result, informal of Theorem~\ref{thm:formal}]\label{thm:informal}
Let $d \geq n$. Let $X \in \R^{n \times d}$. Let $f(z) \in \{ \exp(z), \cosh(z) \}$. Let $r, \epsilon, \delta \in (0,0.1)$. Let $\Delta = 0.1 \min \{ \frac{\epsilon}{ \sqrt{ k \log(1/\delta)} } ,  \frac{\epsilon}{ \log(1/\delta)} \}$. Let $A = {\cal M} (X) = XX^\top$ and $\| A \|_{\infty} \leq r$. 
Let $f(A)$ and $\D(A)$ be defined as Definition~\ref{def:general_dp_attention}.
For all $X$ sampled from $\cal{D}$, $X$ is $(\alpha, \eta)$-good (see Definition~\ref{def:dataset}).
Let $\eta < r$.
Let $\beta$ be the parameter for the neighboring dataset.
Let $ 2 \alpha \beta \sqrt{n} / \eta < \Delta$.
Suppose $\| {\cal M}(X)^{1/2} {\cal M}(\wt{X})^{-1} {\cal M}(X)^{1/2} - I \|_F \leq \Delta$ for all $X \in \R^{n \times d}, \wt{X} \in \R^{n \times d}$  (see Definition~\ref{def:neighboring}). 
Let $\rho = \sqrt{ ( n^2 + \log(1/\gamma) ) / k } + (n^2 + \log(1/\gamma)) / k  < 0.1 \epsilon$.

Then, there is an algorithm (Algorithm~\ref{alg:main}) that takes $X$ as input and produces the matrix $B \in \R^{n \times n}$ and also general attention $\D(B)^{-1} f(B)$ as output such that
\begin{align*}
    \| \D(A)^{-1} f(A) - \D(B)^{-1} f(B) \|_{\infty} \leq 4 \cdot (1+\epsilon +2r) \cdot r
\end{align*}
which holds with probability $1-\gamma$. With respect to $X$, the algorithm is $(\epsilon,\delta)$-differential private.
\end{theorem}

\paragraph{Roadmap.}
Our paper is organized as follows. We discuss related work in Section~\ref{sec:related}.
We provide our preliminary in Section~\ref{sec:preli}.
Our main result is presented in Section~\ref{sec:main_result}.
We provide an overview of our techniques in Section~\ref{sec:tech}. 
In Section~\ref{sec:conclusion}, we give our conclusion of the paper. 

\section{RELATED WORK}\label{sec:related}

\subsection{Attention Mechanism}
Attention mechanisms are foundational in modern neural networks, gaining widespread adoption since their introduction in~\cite{vsp+17}. They are crucial in decoder-only LLMs~\cite{rwc+19} and the Vision Transformer (ViT)~\cite{dbk+20}, driving significant progress in language models and computer vision~\cite{rbl+22, wsd+23, wxz+24, wcz+23,swxl24,lssy24}. Additionally, attention mechanisms have been applied to multi-modal models~\cite{as24_iclr,xgh+21, zhjl24, lssz24_tat, wms+24}, mathematical reasoning~\cite{lls+24_grok,xsl24}, diffusion models~\cite{px23, lssz24_gm, hwsl24}, differential privacy~\cite{bepp22, ssc+22, lssz24_dp, csy23a}, Hopfield models~\cite{hyw+23,whl+24,hlsl24,xhh+24,whhl24,hcl+24,hcw+24}, and various other techniques~\cite{lss+24_relu, lsy24, qszz23, lss+24,lls+24_io,lls+24_prune,cls+24,smn+24,llss24_sparse,xsw+23}.

\subsection{Differential Privacy and Deep Learning}
Differential privacy (DP) is a rigorous and quantifiable notion of privacy that ensures individual data entries cannot be distinguished within a dataset. It has become the go-to standard for understanding information leakage \cite{dr14}. This widely recognized framework is increasingly being adopted in industry and has many real-world applications \cite{xza+23,tm22,s22,f22,a22,rsy+21}. There has been extensive research on applying differential privacy in deep learning \cite{acg+16,kkm+20,ggk+21,lssz24_dp,syyz23,lsss24_dp_ntk,lls+24_dp_je}. Recent works \cite{ynb+21,ltl+21} have applied DP-SGD \cite{acg+16} to large language models (LLMs) for private fine-tuning. Our research, however, is orthogonal to these works as we focus on attention computation and consider general differential privacy mechanisms, not just DP-SGD.

\subsection{Softmax Computation and Regression}
With the rapid development of large language models and attention schemes, many works have focused on softmax computation and regression in this field.
\cite{as23,as24} shows that a faster attention algorithm can be designed by leveraging the matrix implicitly. \cite{bsz23} proposes a more efficient algorithm for computing dynamic attention by employing the method of lazy update. To solve $\exp$, $\cosh$, and $\sinh$ regressions with input sparsity, \cite{lsz23} use an approximate Newton method that operates in near-linear time. In their work on softmax regression, \cite{dls23} conduct a further analysis of attention schemes based on prior research in regression. In contrast, \cite{gms23} focus on the convergence analysis of overparameterized two-layer networks with exponential activation functions.
To compute the attention matrix more efficiently for large feature dimensions, \cite{dms23} propose a randomized algorithm.

\section{PRELIMINARY}\label{sec:preli}

Section~\ref{sec:preli:notations} presents the notations that are used throughout our paper. 
In Section~\ref{sec:preli:useful_tools}, we provide the previous tools that help our proofs. 
\subsection{Notations}\label{sec:preli:notations}

$\mathbb{E}[X]$ represents the expected value (or mean) of a random variable $X$. 
We use $\chi_d^2$ to denote a Chi-squared random variable with $d$ degrees of freedom. 
If $M$ and $N$ are symmetric matrices, we define $M \succeq N$ to mean that for all vectors $x$, the inequality $x^\top M x \geq x^\top N x$ holds.
If $M$ is a symmetric matrix of dimension $n \times n$, we define $M$ to be positive semidefinite ($M \succeq 0$) if the inequality $x^\top M x \geq 0$ holds for all vectors $x \in \R^n$.
We use the notation ${\bf 0}_n$ to denote an $n$-dimensional vector whose entries are all zero, and ${\bf 1}_n$ to denote an $n$-dimensional vector whose entries are all one. The symbol $I_n$ represents the $n \times n$ identity matrix, which is a square matrix with ones on the main diagonal and zeros elsewhere.
Let $x$ be an arbitrary vector in $\R^n$. 
We define $\exp(x) \in \R^n$ as a vector whose $i$-th entry $\exp(x)_i$ is equal to $\exp(x_i)$, where $\exp(\cdot)$ denotes the exponential function. We use $\langle x,y \rangle$ to denote $\sum_{i=1}^n x_i y_i$. 
For any matrix $A$, we use $\| A \|$ to denote the spectral norm of $A$, i.e., $\|A\| = \max_{\|x\|_2 = 1} \|Ax\|_2$, $\| A \|_F$ to denote its Frobenius norm and $\| A \|_\infty$ to denote the infinity norm. $A_{i,j}$ represents the element in the $i$-th row and $j$-th column of matrix $A$. 

\subsection{Previous Tools}\label{sec:preli:useful_tools}

This section introduces several differential privacy tools.
These tools are essential for demonstrating the differential privacy properties of our algorithm.
\begin{theorem}[Empirical covariance estimator for Gaussian \cite{v18}]\label{thm:empirical_covariance_estimator_for_gaussian}
Let $\Sigma \in \R^{d \times d}$ be PSD, $X_1,\cdots,X_n \sim {\cal N}(0,\Sigma)$ be i.i.d and $\wt{\Sigma} = \frac{1}{n} \sum_{i=1}^n X_i X_i^{\top}$. Then with probability $1 - \gamma$, it holds that
$
    \| \Sigma^{-1/2} \wt{\Sigma} \Sigma^{-1/2} - I\|_F \leq \rho
$
for some $\rho = O( \sqrt{ \frac{ d^2 + \log(1/\gamma)}{n}} + \frac{ d^2 + \log(1/\gamma)}{n})$.
\end{theorem}

\begin{theorem}[Lemma 1.5 in \cite{v17}, Section 1.1 of \cite{bs16}]\label{thm:epsilon_delta_DP}
For a (randomized) mechanism ${\cal M}$ and  datasets $x,y$, define the function
$
    f_{xy}(z) := \log ( \frac{ \Pr[{\cal M}(x) = z ]}{\Pr[{\cal M}(y) = z ]})
$
If $\Pr[ f_{xy}( {\cal M}(x)) > \epsilon] \leq \delta$ for all adjacent datasets $x,y$, then ${\cal M}$ is $( \epsilon, \delta)$-DP.
\end{theorem}

\begin{lemma}[Sub-exponential tail bound, Proposition~2.9 in \cite{w19}]\label{lem:sub_exponential_tail_bound}
Suppose that $X$ is sub-exponential with parameters $(\nu, \alpha)$. Then
$
    \Pr[X - \mu \geq t] \leq \max \{ \exp( -\frac{t^2}{2v^2}) , \exp( \frac{t}{2 \alpha})\}.
$
\end{lemma}

\begin{lemma}[$\chi_1^2$ sub-exponential parameters, Example 2.11 in \cite{w19}]\label{lem:sub_exponential_parameters}
    A chi-squared random variable with $1$ degree of freedom $( \chi_1^2)$ is sub-exponential with parameters $(\nu,\alpha) = (2,4)$
\end{lemma}

\begin{lemma}[Sub-exponential parameters of independent sum, Chapter 2 of \cite{w19}]\label{lem:sub_exponential_parameters_independent}
Consider an independent sequence $X_1,\cdots,X_k$ of random variables, such that $X_i$ is sub-exponential with parameters $(\nu_i, \alpha_i)$. Then the variable $\sum_{i=1}^k X_i$ is sub-exponential with parameters $(\nu_*,\alpha_*)$, where $a_* = \max_{i\in [k]} \alpha_i$ and $\nu_* = ( \sum_{i=1}^k \nu_i^2 )^{1/2}$.
\end{lemma}

\section{MAIN RESULT}\label{sec:main_result}

\begin{algorithm}\caption{Differential privacy algorithm}\label{alg:main}
\begin{algorithmic}[1]
\Procedure{DPAttention}{$X$}
    \State $A \gets XX^\top$
    \State $B \gets \textsc{DPCovariance}(A,k)$ \Comment{See Algorithm~\ref{alg:the_gaussian_sampling_mechanism}.}
    \State Compute $f(B)$
    \State Compute $\D(B)^{-1} f(B)$
\EndProcedure
\end{algorithmic}
\end{algorithm}

In this section, we provide a theoretical analysis of Algorithm~\ref{alg:main}, our primary algorithm for differentially private general attention computation. Our analysis leverages the tools established in Section~\ref{sec:short_error}, Section~\ref{sec:short_gaussian}, and Section~\ref{sec:sensitivity}. From our previous proofs, it is evident that our algorithm possesses a rigorous differential privacy property, offering new insights into both differential privacy and attention mechanisms.
\begin{theorem}[Main result]\label{thm:formal}
If all of the following requirements are met
Let $d \geq n$, $X \in \R^{n \times d}$, and $f(z) \in \{ \exp(z), \cosh(z) \}$.
We define $r \in (0,0.1)$ as bounded ratio and $\epsilon, \delta \in (0,0.1)$ as the parameter of DP.
Let $\Delta = 0.1 \min \{ \frac{\epsilon}{ \sqrt{ k \log(1/\delta)} } ,  \frac{\epsilon}{ \log(1/\delta)} \}$.
Let $A = {\cal M} (X) = XX^\top$ and $\| A \|_{\infty} \leq r$.
For all $X$ sampled from $\cal{D}$, $X$ is $(\alpha, \eta)$-good (see Definition~\ref{def:dataset}).
Let $\eta < r$.
Let $\beta$ be the parameter for neighboring dataset.
Let $ 2 \alpha \beta \sqrt{n} / \eta < \Delta$.
Let $\Delta$ denote the sensitivity parameter that ${\cal M}$ satisfies a sensitivity bound that $\| {\cal M}(X)^{1/2} {\cal M}(\wt{X})^{-1} {\cal M}(X)^{1/2} - I \|_F \leq \Delta$ for any neighboring datasets $X \in \R^{n \times d}, \wt{X} \in \R^{n \times d}$  (see Definition~\ref{def:neighboring}). 
Let $\rho = \sqrt{ ( n^2 + \log(1/\gamma) ) / k } + (n^2 + \log(1/\gamma)) / k $ and $\rho < 0.1 \epsilon$.

There is an algorithm (Algorithm~\ref{alg:main}) that takes $X$ as input and produces the matrix $B \in \R^{n \times n}$ and also attention $\D(B)^{-1} f(B)$ as output such that
\begin{itemize}
    \item {\bf Part 1.} 
    $ \| \D(A)^{-1} f(A) - \D(B)^{-1} f(B) \|_{\infty} 
     \leq 4 \cdot (1+\epsilon +2r) \cdot r$.
    \item {\bf Part 2.} With respect to $X$, the algorithm is $(\epsilon,\delta)$-differential private.
    \item {\bf Part 3.} It holds with probability $1-\gamma$.
\end{itemize}
\end{theorem}

\begin{proof}[Proof of Theorem~\ref{thm:formal}]
The proof can be divided into two parts as follows. 
\paragraph{Proof of Part 1 and Part 3.}
Our proof focus on the function ${\cal M}(X) := X X^\top$ first. 
Let $\alpha$ and $\eta$ be denoted in Definition~\ref{def:dataset} and $\beta$ be denoted as Definition~\ref{def:neighboring}. Based on the assumption on dataset above, we can obtain $X$ is $(\eta, \alpha)$-good (See Definition~\ref{def:dataset}) while $X$ and $\wt{X}$ are $\beta$-close (See Definition~\ref{def:neighboring}).

According to {\bf Part 1} of Lemma~\ref{lem:sensitivity_from_spectral_to_F:short}, we can conclude the property on ${\cal M}(X) = X X^\top$ such that
\begin{align*}
     \|  (XX^\top)^{-1/2}  \wt{X} \wt{X}^\top (XX^\top)^{-1/2} - I \|_F \leq 2 \sqrt{n} \alpha \beta/\eta
\end{align*}
Let ${\cal M}$ be the function denoted in the theorem statement and let $\rho$ be denoted as follows:
\begin{align*}
    \rho := O( \sqrt{ ( n^2+\log(1/\gamma) )  / k }+ ( n^2+\log(1/\gamma) ) /{k} )
\end{align*}
Now, we will apply the conclusion drawn in Section~\ref{sec:short_gaussian}. In order to satisfy the requirement specified in {\bf Requirement 4} of Theorem~\ref{thm:analysis_of_the_aussian_sampling_mechanism:informal}, we need ${\cal M}(X)$ to meet the following assumption:
\begin{align*}
\| \mathcal{M}(X)^{1/2}\mathcal{M}(\wt{X})^{-1}\mathcal{M}(X)^{1/2}-I \|_F \leq \Delta.
\end{align*}

Now, if we choose $2 \alpha \beta \sqrt{n} / \eta < \Delta$.
we will guarantee that our ${\cal M}(X)$ satisfies the assumption specified in {\bf Requirement 4} of Theorem~\ref{thm:perturb_psd}. According to {\bf Part 3} of Theorem~\ref{thm:perturb_psd}, there exists Algorithm~\ref{alg:the_gaussian_sampling_mechanism} which can produce a matrix $B \in \R^{n \times n}$ such that, with probability at least $1 - \gamma$
\begin{align}\label{eq:main:A_B}
    (1-\rho) A \preceq B \preceq (1+\rho)  A
\end{align}
By choosing $\rho \in (0,0.1)\epsilon$, we will have
\begin{align}\label{eq:main:A_B_2}
    (1-\epsilon) B \preceq A \preceq (1+\epsilon) B
\end{align}

Now according to Theorem~\ref{thm:perturb_psd} and Eq.~\eqref{eq:main:A_B_2}, we have
\begin{align*}
     \| \D( A )^{-1} f(A) - \D( B )^{-1} f(B) \|_{\infty} \leq 4 \cdot (1 + \epsilon + 2 r) \cdot r
\end{align*}

Now, the proofs of {\bf Part 1} and {\bf Part 3} are completed.

\paragraph{Proof of Part 2.}
It simply follows from { \bf Part 1} of Theorem~\ref{thm:analysis_of_the_aussian_sampling_mechanism:informal}.
\end{proof}

The main result implies that we can design an algorithm that computes a private approximation of the attention mechanism used in neural networks for functions like $ f(z) = \exp(z) $ or $ f(z) = \cosh(z) $. Specifically, under certain conditions on the input matrix $X$ and parameters $ \epsilon, \delta $, and with a small bounded ratio $ r $, the algorithm produces a matrix $ B $ such that the normalized attention matrices derived from $ A = XX^\top $ and $ B $ are close in the infinity norm. This closeness is quantified by a bound proportional to $ r $, ensuring that the utility of the attention mechanism is preserved. Additionally, the algorithm is $(\epsilon, \delta)$-differentially private with respect to $ X $, meaning it protects individual data entries from being inferred. The privacy and utility guarantees hold with high probability $ 1 - \gamma $, demonstrating that it is possible to implement attention mechanisms in a way that maintains both model performance and data privacy.

\section{TECHNIQUE OVERVIEW}\label{sec:tech}

The objective of our research is to develop a differentially private algorithm that addresses the challenges of computing attention on large datasets. Specifically, we focus on scenarios where the size of the data matrix $X$ is extremely large, with the number of features $d$ significantly exceeding the number of samples $n$ (i.e., $d \gg n$). In these cases, the attention matrix $A$ is obtained as the output of the function ${\cal M}(X)=XX^\top$, and our goal is to ensure that the computation of $A$ is performed in a differentially private \cite{dmns06,dkm+06} manner.

\paragraph{Perturb PSD Matrix.}
We define the attention computation $\D(X)$ as Definition~\ref{def:D}.
By employing a more general version of Perturbation analysis presented in \cite{dms23}, we select $f$ as specified in Definition~\ref{def:f}.
To complete the error analysis of attention computation, we will utilize the perturbation analysis of the diagonal normalization matrix and the PSD matrix presented in Appendix~\ref{sec:error_control}.
Under the assumption the relative error between input matrix $ {\cal M}(X)  := A$ and privacy required matrix output $B$ is less than or equal to $\epsilon \in (0, 0.1)$ where $(1-\epsilon)B \preceq A \preceq (1+\epsilon) B$.
To establish an upper bound for $\| \D( A )^{-1} f(A) - \D( B )^{-1} f(B) \|_{\infty}$, we first derive the following bound:
\begin{itemize}
    \item {\bf Part 1.} 
            $| \D(A)_{i,i} - \D(B)_{i,i} | \leq c_1 \cdot r \cdot \min \{ \D(A)_{i,i}, \D(B)_{i,i} \}~~ \forall i \in [n]$,
    \item {\bf Part 2.} 
            $| f( A_{i,j} ) - f( B_{i,j} ) | \leq c_2 \cdot r \cdot \min\{ f(A_{i,j} ), f( B_{i,j} ) \} ~~ \forall  i,j \in [n] \times [n]$
    \end{itemize}
And with the error of attention computation under control as mentioned above, we can obtain:
\begin{align*}
     \| \D( A )^{-1} f(A) - \D( B )^{-1} f(B) \|_{\infty} \leq 4 \cdot (1 + \epsilon + 2 r) \cdot r
\end{align*}

\paragraph{Sensitivity for PSD Matrix.}
Our work relies on the basic assumptions that $X \in \R^{n \times d}$ is a $(\eta,\alpha)$-good dataset (See Definition~\ref{def:dataset}) and that $X$ and $\wt{X}$ are $\beta$-close to each other (See Definition~\ref{def:neighboring}). We choose $\mathcal{M}(X) := X X^\top$. Now we will demonstrate the property of our function ${\cal M}(X) = X X^\top$ based on the given assumptions. Since $X$ and $\wt{X}$ are neighbor datasets, we have the following:
\begin{align*}
    \| {\cal M}(X)^{1/2} {\cal M}(\wt{X})^{-1} {\cal M}(X)^{1/2} - I \|_F \leq 2 \alpha \beta \sqrt{n}
\end{align*}
The proof details can be found in Section~\ref{sec:app_sensitivity}. Let us denote $\Delta$ as defined in Definition~\ref{def:delta}. By choosing $2 \alpha \beta \sqrt{n} / \eta < \Delta$, we will have
\begin{align}\label{eq:tech:assumption_M}
    \| (\underbrace{X X^\top}_{: = {\cal M}(X)})^{1/2} (\underbrace{\wt{X} \wt{X}^\top}_{:={\cal M}(\wt{X})})^{-1} (\underbrace{X X^\top}_{: = {\cal M}(X)})^{1/2} - I \|_F \leq \Delta
\end{align}
The assumption specified in the {\bf Requirement 5} of Theorem~\ref{thm:analysis_of_the_aussian_sampling_mechanism} will be satisfied. Next, we will introduce our main algorithm using Eq.~\eqref{eq:tech:assumption_M}.
\paragraph{Differential Privacy Algorithm.}
Next we give the differential privacy algorithm described in Theorem~\ref{thm:informal}. And we will demonstrate that our algorithm (Algorithm~\ref{alg:main}) is able to output a matrix that satisfies the {\bf Part 1} of our formal main result (See Theorem~\ref{thm:formal}).

To begin with, we demonstrate that there exists an algorithm capable of taking input $A$ and producing a matrix $B$ as output such that the difference between $A$ and $B$ is small enough, which can be seen as a small error resulting from the perturbation of $A$ by 
\begin{align*}
    \rho : = O(\sqrt{(n^2 + \log(1/\gamma))/k} + (n^2 + \log(1/\gamma))/k).
\end{align*} In other words, we have
$(1-\rho) A \preceq B \preceq (1+\rho) A$.
The above equation holds with probability $1 - \gamma$. Note that $k$ and $\gamma$ can be chosen according to our requirements. We can ensure that a satisfactory $\rho$ is obtained. By choosing a small enough $\rho \leq 0.1 \epsilon$ and using the conclusions on perturbed PSD matrices, the algorithm can certainly output a satisfactory $B$ which promises our attention computation is privacy \cite{dmns06,dkm+06}.

\subsection{Error Control from Logit Matrix to Attention Matrix}\label{sec:short_error}

In this section, we analyze the perturbations in the attention computation, which are used to control the error.
First, we define the followings.
\begin{definition}\label{def:f}
Let $f(z)$ denote one of the following functions
$\exp(z)$ and $\cosh(z)$.
\end{definition}
The motivation of considering $\exp(z)$ is due to recent LLMs. The motivation of considering $\cosh(z)$ is from recent progress in potential function design of convex optimization \cite{cls19,lsz19,song19,b20,jswz21,gs22,qszz23}.

\begin{definition}\label{def:D}
Given that $A \in \R^{n \times n}$, we define $f$ as Definition~\ref{def:f}. Let us define
$
    \D(A):= \diag( f(A) {\bf 1}_n )
$
where we apply $f$ to matrix entrywisely.
\end{definition}

We state a major tool we proved in this paper to control the error propagation which summarizes the effectiveness of our error control mechanisms in achieving differential privacy for the computation of the attention matrix.

\begin{theorem}\label{thm:perturb_psd}

Let $\epsilon \in (0,0.1)$ and $r \in (0,0.1)$.
Let $\| A \|_{\infty} \leq r$ and $(1-\epsilon) B \preceq A \preceq (1+\epsilon) B$.
We define $\D$  as Definition~\ref{def:D} and $f$ as Definition~\ref{def:f}.

Then, we have
\begin{align*}
    \| \D( A )^{-1} f(A) - \D( B )^{-1} f(B) \|_{\infty} \leq 4 \cdot (1 + \epsilon + 2 r) \cdot r
\end{align*}
\end{theorem}

The prior work \cite{dms23} only work for $\exp()$ function and the final guarantee is $O(r)$. We generalize it to $\cosh()$ function also, and our error bound is much tighter. The proof of the theorem above is in Section~\ref{sec:error_control:main_result}.

\subsection{Analysis of Gaussian Sampling Mechanism}\label{sec:short_gaussian}

This section introduces a crucial component of our main differential privacy algorithm (Algorithm~\ref{alg:main}): the differentially private covariance releasing mechanism, detailed in Algorithm~\ref{alg:the_gaussian_sampling_mechanism}.
The differential privacy (DP) of Algorithm~\ref{alg:the_gaussian_sampling_mechanism} ensures the DP of the main algorithm (Algorithm~\ref{alg:main}). Therefore, we will also demonstrate its DP.
Our proof is based on the assumption that the sensitivity is bounded (Requirement 4 in Theorem~\ref{thm:analysis_of_the_aussian_sampling_mechanism:informal}). We defer the validation of the assumption to Section~\ref{sec:sensitivity}.
For clarity, the following proof is based on the assumption that $M \leq \Delta$ (See Definition~\ref{def:m} and Definition~\ref{def:delta}), which will be proven in Section~\ref{sec:sensitivity}. 
Let ${\cal Y}$ and ${\cal Y'}$ be neighboring datasets, as denoted in Definition~\ref{def:neighboring}.

To facilitate the explanation in the following proof, we will define $M$ to better illustrate the properties of ${\cal M}$.
\begin{definition}\label{def:m}
    Let  $\mathcal{M}: (\R^n)^d \to \R^{n \times n}$ be a (randomized) algorithm that given a dataset of $d$ points in $\R^n$ outputs a PSD matrix. Then, we define 
    \begin{align*}
        M := \sup_{ {\cal Y} , {\cal Y'} } \| \mathcal{M}(\mathcal{Y})^{1/2}\mathcal{M}(\mathcal{Y}')^{-1}\mathcal{M}(\mathcal{Y})^{1/2}-I  \|_F 
    \end{align*}
    Here $\sup$ is over all neighboring datasets ${\cal Y}$ and ${\cal Y}'$ (see Definition~\ref{def:neighboring}).
\end{definition}

We define the upper bound of $M$ as $\Delta$ as follows.
\begin{definition}\label{def:delta}
    Let $M$ be defined in Definition~\ref{def:m}. We define
    $
        \Delta := \min\big\{ \frac{\epsilon}{\sqrt{8k \log(1/\delta)}},\frac{\epsilon}{8 \log(1/\delta)} \big\} 
    $
    such that $M \leq \Delta$.
\end{definition}

We define the notation of gaussian distribution below.
\begin{definition}[Gaussian Distribution]\label{def:gaussian_density}
 We denote the $\mathcal{N}(0,\Sigma)$ density function as follows
    \begin{align*}
        f_\Sigma(x) = (2 \pi)^{-\frac{n}{2}} \det(\Sigma)^{-\frac{1}{2}} \exp (- 0.5 x^\top \Sigma x )
    \end{align*}
\end{definition}

We state our differentially private covariance releasing algorithm.
Assuming $g_i$s are i.i.d. samples from $f_{\Sigma}(x)$ for $i \in [k]$, we use $g_1, g_2,\cdots,g_k$ to compute the covariance estimate $\hat{\Sigma}$ in Algorithm~\ref{alg:the_gaussian_sampling_mechanism}. We will demonstrate the analysis of $\hat{\Sigma}=\frac{1}{k}\sum^k_{i=1}g_i g_i^\top$ using the symbol from Appendix~\ref{sec:gaussian_sampling_mechanism}, which leads to the privacy guarantee for our algorithm~\ref{alg:the_gaussian_sampling_mechanism}.
\begin{algorithm}
    \caption{Differentially private covariance releasing}\label{alg:the_gaussian_sampling_mechanism}
    \begin{algorithmic}[1]
        \Procedure{DPCovariance}{$\Sigma \in \R^{n \times n}$, $k\in \mathbb{N}$}
        \Comment{PSD matrix $\Sigma$ and parameter $k $}
        \State Obtain vectors $g_1,g_2,\cdots,g_k$ by sampling $g_i \sim \mathcal{N}(0,\Sigma)$, independently for each $i \in [k]$
        \State Compute $\hat{\Sigma}=\frac{1}{k}\sum_{i=1}^k g_i g_i^\top$ \Comment{This is Covariance estimate.}
        \State \Return $\hat{\Sigma}$ 
        \EndProcedure
    \end{algorithmic}
\end{algorithm}

The soundness Algorithm~\ref{alg:the_gaussian_sampling_mechanism} can be shown using Theorem~\ref{thm:analysis_of_the_aussian_sampling_mechanism:informal}.
We now give the definitions of $\Sigma_1,\Sigma_2$, $h_{i,j}$ and $Z$ which will be used to prove the Theorem~\ref{thm:analysis_of_the_aussian_sampling_mechanism:informal}.
\begin{definition}\label{def:sigma}
Let ${\cal M}$ be denoted in Definition~\ref{def:m} and $\Sigma ({\cal Y}) := {\cal M}({\cal Y})$. We define $\Sigma_1 :=\Sigma(\mathcal{Y})$, $\Sigma_2 :=\Sigma(\mathcal{Y^{'}})$.
\end{definition}

\begin{definition}\label{def:z}
Let $g_1,g_2,\cdots,g_k$ be i.i.d samples from $\mathcal{N}(0,\Sigma_1)$ output by Algorithm~\ref{alg:the_gaussian_sampling_mechanism}. Then, we define $h_{i,j} :=  \langle \Sigma_1^{-1/2}g_i, v_j  \rangle$, $Z := \sum_{i=1}^k \log ( \frac{f_{\Sigma_1}(g_i)}{f_{\Sigma_2}(g_i)} )$ where $\Sigma_1,\Sigma_2$ are defined by Definition~\ref{def:sigma}. 
Note that the random variables $h_{i,j}$ are i.i.d copies of $\mathcal{N}(0,1)$.
\end{definition}

We will now present our theorem for Algorithm~\ref{alg:the_gaussian_sampling_mechanism}. The proof is delayed to Section~\ref{sec:gaussian_sampling_mechanism:main_result}.
\begin{theorem}[Informal version of Theorem~\ref{thm:analysis_of_the_aussian_sampling_mechanism}]\label{thm:analysis_of_the_aussian_sampling_mechanism:informal}
    If all of the following requirements are met: {\bf Requirement 1.} Let $\epsilon \in (0,1)$ and $\delta \in (0,1)$. {\bf Requirement 2.} $k \in \mathbb{N}$.  {\bf Requirement 3.} Let $\Delta$ be denoted as Definition~\ref{def:delta} and $ \Delta < 1$.    {\bf Requirement 4.} Let $M,{\cal M}$ be denoted as Definition~\ref{def:m} and $M \leq \Delta$.  {\bf Requirement 5.} An input $\Sigma = \mathcal{M}(\mathcal{Y})$.  {\bf Requirement 6.} $\rho = O( \sqrt{ ( n^2+\log(1/\gamma) )  / k }+ ( n^2+\log(1/\gamma) ) /{k} )$.
    
    Then, there is an algorithm (Algorithm~\ref{alg:the_gaussian_sampling_mechanism}) such that
    \begin{itemize}
        \item Part 1. Algorithm~\ref{alg:the_gaussian_sampling_mechanism} is $(\epsilon,\delta)$-DP (with respect to the original dataset $\mathcal{Y}$).
        \item Part 2. outputs $\hat{\Sigma} \in \mathbb{S}_+^n$ such that with probability at least $1-\gamma$, $\| \Sigma^{-1/2} \wh{\Sigma} \Sigma^{-1/2}-I_n \|_F \leq \rho$
        \item Part 3. $(1-\rho) \Sigma \preceq \wh{\Sigma} \preceq (1+\rho)  \Sigma$.
    \end{itemize}
\end{theorem}

Using this theorem, we can see our Algorithm~\ref{alg:the_gaussian_sampling_mechanism} is DP, which ensures the DP of Algorithm~\ref{alg:main}.

\subsection{Sensitivity for PSD Matrix}\label{sec:sensitivity}

We have demonstrated the existence of a differential privacy algorithm under the assumption on ${\cal M}(X)=X X^\top $ introduced in Section~\ref{sec:short_gaussian}. In this section,
we show that ${\cal M}(X) = X X^\top$ satisfies the assumption specified in {\bf Requirement 4} of Theorem~\ref{thm:analysis_of_the_aussian_sampling_mechanism:informal} for ${\cal M}(X)$. The lemma following is based on the assumption on datasets $X,\wt{X}$ (See Definition~\ref{def:dataset} and Definition~\ref{def:neighboring}).
\begin{lemma}[Informal version of Lemma~\ref{lem:sensitivity_formal}]\label{lem:sensitivity}
   If  $X \in \R^{n \times d}$ and $\wt{X} \in \R^{n \times d}$ are neighboring dataset (see Definition~\ref{def:dataset} and Definition~\ref{def:neighboring}), then
   $
         (1- 2 \alpha \beta/\eta)  X X^\top \preceq \wt{X} \wt{X}^\top \preceq (1+ 2 \alpha \beta/\eta) XX^\top
    $.
\end{lemma}

Now, we can have the following lemma. The subsequent lemma can be viewed as a variation of Lemma~\ref{lem:sensitivity}, yet it presents a more apparent result that can be directly employed in subsequent analyses. 
\begin{lemma}\label{lem:sensitivity_from_spectral_to_F:short}
Let $\alpha$ and $\eta$ be defined in Definition~\ref{def:dataset}.
Let $\beta$ be defined in Definition~\ref{def:neighboring}. 
Let $X$ and $\wt{X}$ be neighboring datasets such that
$
  (1-2 \alpha \beta/\eta)  X X^\top \preceq \wt{X} \wt{X}^\top \preceq (1+2 \alpha \beta/\eta) XX^\top.
$

Then, we have 
\begin{itemize}
\item Part 1. 
$\|  (XX^\top)^{-1/2}  \wt{X} \wt{X}^\top (XX^\top)^{-1/2} - I \| \leq 2 \alpha \beta/\eta$
\item Part 2. 
$ \|  (XX^\top)^{-1/2}  \wt{X} \wt{X}^\top (XX^\top)^{-1/2} - I \|_F \leq 2 \sqrt{n} \alpha \beta/\eta$
\end{itemize}
\end{lemma}

The proof of Lemma~\ref{lem:sensitivity_from_spectral_to_F:short} follows directly from the Lemma~\ref{lem:sensitivity}. 

Presently, we have delved into the sensitivity property of attention computation. We are able to illustrate that the computation of the attention matrix aligns with the assumptions introduced in Section~\ref{sec:short_gaussian}. Building upon this foundation, we will subsequently address our primary result in Section~\ref{sec:main_result}.

\section{CONCLUSION}\label{sec:conclusion}

In this work, we propose a differentially private algorithm for approximating the attention matrix.
Our algorithm is built upon recent advances in fast attention computation and private matrix releasing.
To the best of our knowledge, this is the first work of accelerating attention computation in the DP setting.
Given the dominating presence of Transformer based language models, we hope our work can stand as a starting point for fully DP training and inferring algorithms on large language models. 
It is also an interesting open problem to approximate asymmetric attention computation with differential privacy.

\ifdefined\isarxiv
\bibliographystyle{alpha}
\bibliography{ref}
\else
\bibliography{ref}

\newcommand{\etalchar}[1]{$^{#1}$}
\begin{thebibliography}{KGW{\etalchar{+}}23}

\bibitem[ACG{\etalchar{+}}16]{acg+16}
Martin Abadi, Andy Chu, Ian Goodfellow, H~Brendan McMahan, Ilya Mironov, Kunal
  Talwar, and Li~Zhang.
\newblock Deep learning with differential privacy.
\newblock In {\em Proceedings of the 2016 ACM SIGSAC conference on computer and
  communications security}, pages 308--318, 2016.

\bibitem[AKT{\etalchar{+}}22]{akt+22}
Daniel Alabi, Pravesh~K Kothari, Pranay Tankala, Prayaag Venkat, and Fred
  Zhang.
\newblock Privately estimating a gaussian: Efficient, robust and optimal.
\newblock {\em arXiv preprint arXiv:2212.08018}, 2022.

\bibitem[AS23]{as23}
Josh Alman and Zhao Song.
\newblock Fast attention requires bounded entries.
\newblock {\em arXiv preprint arXiv:2302.13214}, 2023.

\bibitem[AS24a]{as24}
Josh Alman and Zhao Song.
\newblock The fine-grained complexity of gradient computation for training
  large language models.
\newblock {\em arXiv preprint arXiv:2402.04497}, 2024.

\bibitem[AS24b]{as24_iclr}
Josh Alman and Zhao Song.
\newblock How to capture higher-order correlations? generalizing matrix softmax
  attention to kronecker computation.
\newblock In {\em The Twelfth International Conference on Learning
  Representations}, 2024.

\bibitem[BEPP22]{bepp22}
Rouzbeh Behnia, Mohammadreza~Reza Ebrahimi, Jason Pacheco, and Balaji
  Padmanabhan.
\newblock Ew-tune: A framework for privately fine-tuning large language models
  with differential privacy.
\newblock In {\em 2022 IEEE International Conference on Data Mining Workshops
  (ICDMW)}, pages 560--566. IEEE, 2022.

\bibitem[BKO18]{bko18}
Mayank Bansal, Alex Krizhevsky, and Abhijit Ogale.
\newblock Chauffeurnet: Learning to drive by imitating the best and
  synthesizing the worst.
\newblock {\em arXiv preprint arXiv:1812.03079}, 2018.

\bibitem[BMR{\etalchar{+}}20]{bmr+20}
Tom Brown, Benjamin Mann, Nick Ryder, Melanie Subbiah, Jared~D Kaplan, Prafulla
  Dhariwal, Arvind Neelakantan, Pranav Shyam, Girish Sastry, Amanda Askell,
  et~al.
\newblock Language models are few-shot learners.
\newblock {\em Advances in neural information processing systems},
  33:1877--1901, 2020.

\bibitem[Bra20]{b20}
Jan van~den Brand.
\newblock A deterministic linear program solver in current matrix
  multiplication time.
\newblock In {\em Proceedings of the Fourteenth Annual ACM-SIAM Symposium on
  Discrete Algorithms (SODA)}, pages 259--278. SIAM, 2020.

\bibitem[BS16]{bs16}
Mark Bun and Thomas Steinke.
\newblock Concentrated differential privacy: Simplifications, extensions, and
  lower bounds.
\newblock In {\em Theory of Cryptography: 14th International Conference, TCC
  2016-B, Beijing, China, October 31-November 3, 2016, Proceedings, Part I},
  pages 635--658. Springer, 2016.

\bibitem[BSZ23]{bsz23}
Jan van~den Brand, Zhao Song, and Tianyi Zhou.
\newblock Algorithm and hardness for dynamic attention maintenance in large
  language models.
\newblock {\em arXiv preprint arXiv:2304.02207}, 2023.

\bibitem[CLS19]{cls19}
Michael~B Cohen, Yin~Tat Lee, and Zhao Song.
\newblock Solving linear programs in the current matrix multiplication time.
\newblock In {\em STOC}, 2019.

\bibitem[CLS{\etalchar{+}}24]{cls+24}
Bo~Chen, Yingyu Liang, Zhizhou Sha, Zhenmei Shi, and Zhao Song.
\newblock Hsr-enhanced sparse attention acceleration, 2024.

\bibitem[CSY23]{csy23a}
Timothy Chu, Zhao Song, and Chiwun Yang.
\newblock How to protect copyright data in optimization of large language
  models?
\newblock {\em arXiv preprint arXiv:2308.12247}, 2023.

\bibitem[DBK{\etalchar{+}}20]{dbk+20}
Alexey Dosovitskiy, Lucas Beyer, Alexander Kolesnikov, Dirk Weissenborn,
  Xiaohua Zhai, Thomas Unterthiner, Mostafa Dehghani, Matthias Minderer, Georg
  Heigold, Sylvain Gelly, et~al.
\newblock An image is worth 16x16 words: Transformers for image recognition at
  scale.
\newblock {\em arXiv preprint arXiv:2010.11929}, 2020.

\bibitem[DCLT18]{dclt18}
Jacob Devlin, Ming-Wei Chang, Kenton Lee, and Kristina Toutanova.
\newblock Bert: Pre-training of deep bidirectional transformers for language
  understanding.
\newblock {\em arXiv preprint arXiv:1810.04805}, 2018.

\bibitem[{Dif}22]{a22}
{Differential Privacy Team, Apple}.
\newblock Learning with privacy at scale.
\newblock
  \url{https://docs-assets.developer.apple.com/ml-research/papers/learning-with-privacy-at-scale.pdf},
  2022.
\newblock Online; accessed 30 November 2022.

\bibitem[DKM{\etalchar{+}}06]{dkm+06}
Cynthia Dwork, Krishnaram Kenthapadi, Frank McSherry, Ilya Mironov, and Moni
  Naor.
\newblock Our data, ourselves: Privacy via distributed noise generation.
\newblock In {\em Annual International Conference on the Theory and
  Applications of Cryptographic Techniques}, pages 486--503. Springer, 2006.

\bibitem[DLS23]{dls23}
Yichuan Deng, Zhihang Li, and Zhao Song.
\newblock Attention scheme inspired softmax regression.
\newblock {\em arXiv preprint arXiv:2304.10411}, 2023.

\bibitem[DMNS06]{dmns06}
Cynthia Dwork, Frank McSherry, Kobbi Nissim, and Adam Smith.
\newblock Calibrating noise to sensitivity in private data analysis.
\newblock In {\em Theory of Cryptography: Third Theory of Cryptography
  Conference, TCC 2006, New York, NY, USA, March 4-7, 2006. Proceedings 3},
  pages 265--284. Springer, 2006.

\bibitem[DMS23]{dms23}
Yichuan Deng, Sridhar Mahadevan, and Zhao Song.
\newblock Randomized and deterministic attention sparsification algorithms for
  over-parameterized feature dimension.
\newblock {\em arxiv preprint: arxiv 2304.03426}, 2023.

\bibitem[DR{\etalchar{+}}14]{dr14}
Cynthia Dwork, Aaron Roth, et~al.
\newblock The algorithmic foundations of differential privacy.
\newblock {\em Foundations and Trends{\textregistered} in Theoretical Computer
  Science}, 9(3--4):211--407, 2014.

\bibitem[EMM{\etalchar{+}}23]{emm+23}
Alessandro Epasto, Jieming Mao, Andres~Munoz Medina, Vahab Mirrokni, Sergei
  Vassilvitskii, and Peilin Zhong.
\newblock Differentially private continual releases of streaming frequency
  moment estimations.
\newblock {\em arXiv preprint arXiv:2301.05605}, 2023.

\bibitem[ERLD17]{erld17}
Javid Ebrahimi, Anyi Rao, Daniel Lowd, and Dejing Dou.
\newblock Hotflip: White-box adversarial examples for text classification.
\newblock {\em arXiv preprint arXiv:1712.06751}, 2017.

\bibitem[{Fac}22]{f22}
{Facebook}.
\newblock Protecting privacy in facebook mobility data during the covid-19
  response, 2022.

\bibitem[FJR15]{fjr15}
Matt Fredrikson, Somesh Jha, and Thomas Ristenpart.
\newblock Model inversion attacks that exploit confidence information and basic
  countermeasures.
\newblock In {\em Proceedings of the 22nd ACM SIGSAC conference on computer and
  communications security}, pages 1322--1333, 2015.

\bibitem[GGK{\etalchar{+}}21]{ggk+21}
Badih Ghazi, Noah Golowich, Ravi Kumar, Pasin Manurangsi, and Chiyuan Zhang.
\newblock Deep learning with label differential privacy.
\newblock {\em Advances in neural information processing systems},
  34:27131--27145, 2021.

\bibitem[GMS23]{gms23}
Yeqi Gao, Sridhar Mahadevan, and Zhao Song.
\newblock An over-parameterized exponential regression.
\newblock {\em arXiv preprint arXiv:2303.16504}, 2023.

\bibitem[GS22]{gs22}
Yuzhou Gu and Zhao Song.
\newblock A faster small treewidth sdp solver.
\newblock {\em arXiv preprint arXiv:2211.06033}, 2022.

\bibitem[GSY23]{gsy23}
Yeqi Gao, Zhao Song, and Junze Yin.
\newblock An iterative algorithm for rescaled hyperbolic functions regression.
\newblock {\em arXiv preprint arXiv:2305.00660}, 2023.

\bibitem[HCL{\etalchar{+}}24]{hcl+24}
Jerry Yao-Chieh Hu, Pei-Hsuan Chang, Haozheng Luo, Hong-Yu Chen, Weijian Li,
  Wei-Po Wang, and Han Liu.
\newblock Outlier-efficient hopfield layers for large transformer-based models.
\newblock In {\em Forty-first International Conference on Machine Learning
  (ICML)}, 2024.

\bibitem[HCW{\etalchar{+}}24]{hcw+24}
Jerry Yao-Chieh Hu, Bo-Yu Chen, Dennis Wu, Feng Ruan, and Han Liu.
\newblock Nonparametric modern hopfield models.
\newblock {\em arXiv preprint arXiv:2404.03900}, 2024.

\bibitem[HLSL24]{hlsl24}
Jerry Yao-Chieh Hu, Thomas Lin, Zhao Song, and Han Liu.
\newblock On computational limits of modern hopfield models: A fine-grained
  complexity analysis.
\newblock In {\em Forty-first International Conference on Machine Learning
  (ICML)}, 2024.

\bibitem[HWL21]{hwl21}
Weihua He, Yongyun Wu, and Xiaohua Li.
\newblock Attention mechanism for neural machine translation: A survey.
\newblock In {\em 2021 IEEE 5th Information Technology, Networking, Electronic
  and Automation Control Conference (ITNEC)}, volume~5, pages 1485--1489. IEEE,
  2021.

\bibitem[HWSL24]{hwsl24}
Jerry Yao-Chieh Hu, Weimin Wu, Zhao Song, and Han Liu.
\newblock On statistical rates and provably efficient criteria of latent
  diffusion transformers (dits).
\newblock {\em arXiv preprint arXiv:2407.01079}, 2024.

\bibitem[HYW{\etalchar{+}}23]{hyw+23}
Jerry Yao-Chieh Hu, Donglin Yang, Dennis Wu, Chenwei Xu, Bo-Yu Chen, and Han
  Liu.
\newblock On sparse modern hopfield model.
\newblock In {\em Thirty-seventh Conference on Neural Information Processing
  Systems (NeurIPS)}, 2023.

\bibitem[JSWZ21]{jswz21}
Shunhua Jiang, Zhao Song, Omri Weinstein, and Hengjie Zhang.
\newblock Faster dynamic matrix inverse for faster lps.
\newblock In {\em STOC}, 2021.

\bibitem[KGW{\etalchar{+}}23]{kgw+23}
John Kirchenbauer, Jonas Geiping, Yuxin Wen, Jonathan Katz, Ian Miers, and Tom
  Goldstein.
\newblock A watermark for large language models.
\newblock {\em arXiv preprint arXiv:2301.10226}, 2023.

\bibitem[KKM{\etalchar{+}}20]{kkm+20}
Sai~Praneeth Karimireddy, Satyen Kale, Mehryar Mohri, Sashank Reddi, Sebastian
  Stich, and Ananda~Theertha Suresh.
\newblock Scaffold: Stochastic controlled averaging for federated learning.
\newblock In {\em International Conference on Machine Learning}, pages
  5132--5143. PMLR, 2020.

\bibitem[KNK21]{knk21}
Oliver Kroemer, Scott Niekum, and George Konidaris.
\newblock A review of robot learning for manipulation: Challenges,
  representations, and algorithms.
\newblock {\em The Journal of Machine Learning Research}, 22(1):1395--1476,
  2021.

\bibitem[LLS{\etalchar{+}}24a]{lls+24_grok}
Chenyang Li, Yingyu Liang, Zhenmei Shi, Zhao Song, and Tianyi Zhou.
\newblock Fourier circuits in neural networks: Unlocking the potential of large
  language models in mathematical reasoning and modular arithmetic.
\newblock {\em arXiv preprint arXiv:2402.09469}, 2024.

\bibitem[LLS{\etalchar{+}}24b]{lls+24_dp_je}
Xiaoyu Li, Yingyu Liang, Zhenmei Shi, Zhao Song, and Junwei Yu.
\newblock Fast john ellipsoid computation with differential privacy
  optimization.
\newblock {\em arXiv preprint arXiv:2408.06395}, 2024.

\bibitem[LLS{\etalchar{+}}24c]{lls+24_io}
Xiaoyu Li, Yingyu Liang, Zhenmei Shi, Zhao Song, and Yufa Zhou.
\newblock Fine-grained attention i/o complexity: Comprehensive analysis for
  backward passes, 2024.

\bibitem[LLS{\etalchar{+}}24d]{lls+24_prune}
Yingyu Liang, Jiangxuan Long, Zhenmei Shi, Zhao Song, and Yufa Zhou.
\newblock Beyond linear approximations: A novel pruning approach for attention
  matrix, 2024.

\bibitem[LLSS24]{llss24_sparse}
Xiaoyu Li, Yingyu Liang, Zhenmei Shi, and Zhao Song.
\newblock A tighter complexity analysis of sparsegpt.
\newblock {\em arXiv preprint arXiv:2408.12151}, 2024.

\bibitem[LOG{\etalchar{+}}19]{log+19}
Yinhan Liu, Myle Ott, Naman Goyal, Jingfei Du, Mandar Joshi, Danqi Chen, Omer
  Levy, Mike Lewis, Luke Zettlemoyer, and Veselin Stoyanov.
\newblock Roberta: A robustly optimized bert pretraining approach.
\newblock {\em arXiv preprint arXiv:1907.11692}, 2019.

\bibitem[LSS{\etalchar{+}}24a]{lss+24_relu}
Yingyu Liang, Zhizhou Sha, Zhenmei Shi, Zhao Song, and Yufa Zhou.
\newblock Looped relu mlps may be all you need as practical programmable
  computers, 2024.

\bibitem[LSS{\etalchar{+}}24b]{lss+24}
Yingyu Liang, Zhizhou Sha, Zhenmei Shi, Zhao Song, and Yufa Zhou.
\newblock Multi-layer transformers gradient can be approximated in almost
  linear time.
\newblock {\em arXiv preprint arXiv:2408.13233}, 2024.

\bibitem[LSSS24]{lsss24_dp_ntk}
Yingyu Liang, Zhizhou Sha, Zhenmei Shi, and Zhao Song.
\newblock Differential privacy mechanisms in neural tangent kernel regression.
\newblock {\em arXiv preprint arXiv:2407.13621}, 2024.

\bibitem[LSSY24]{lssy24}
Yingyu Liang, Zhenmei Shi, Zhao Song, and Chiwun Yang.
\newblock Toward infinite-long prefix in transformer.
\newblock {\em arXiv preprint arXiv:2406.14036}, 2024.

\bibitem[LSSZ24a]{lssz24_dp}
Yingyu Liang, Zhenmei Shi, Zhao Song, and Yufa Zhou.
\newblock Differential privacy of cross-attention with provable guarantee.
\newblock {\em arXiv preprint arXiv:2407.14717}, 2024.

\bibitem[LSSZ24b]{lssz24_tat}
Yingyu Liang, Zhenmei Shi, Zhao Song, and Yufa Zhou.
\newblock Tensor attention training: Provably efficient learning of
  higher-order transformers.
\newblock {\em arXiv preprint arXiv:2405.16411}, 2024.

\bibitem[LSSZ24c]{lssz24_gm}
Yingyu Liang, Zhenmei Shi, Zhao Song, and Yufa Zhou.
\newblock Unraveling the smoothness properties of diffusion models: A gaussian
  mixture perspective.
\newblock {\em arXiv preprint arXiv:2405.16418}, 2024.

\bibitem[LSY24]{lsy24}
Xiaoyu Li, Zhao Song, and Junwei Yu.
\newblock Quantum speedups for approximating the john ellipsoid.
\newblock {\em arXiv preprint arXiv:2408.14018}, 2024.

\bibitem[LSZ19]{lsz19}
Yin~Tat Lee, Zhao Song, and Qiuyi Zhang.
\newblock Solving empirical risk minimization in the current matrix
  multiplication time.
\newblock In {\em Conference on Learning Theory (COLT)}, pages 2140--2157.
  PMLR, 2019.

\bibitem[LSZ23]{lsz23}
Zhihang Li, Zhao Song, and Tianyi Zhou.
\newblock Solving regularized exp, cosh and sinh regression problems.
\newblock {\em arXiv preprint, 2303.15725}, 2023.

\bibitem[LTLH21]{ltl+21}
Xuechen Li, Florian Tramer, Percy Liang, and Tatsunori Hashimoto.
\newblock Large language models can be strong differentially private learners.
\newblock {\em arXiv preprint arXiv:2110.05679}, 2021.

\bibitem[Ope23]{o23}
OpenAI.
\newblock Gpt-4 technical report.
\newblock {\em arXiv preprint arXiv:2303.08774}, 2023.

\bibitem[PX23]{px23}
William Peebles and Saining Xie.
\newblock Scalable diffusion models with transformers.
\newblock In {\em Proceedings of the IEEE/CVF International Conference on
  Computer Vision}, pages 4195--4205, 2023.

\bibitem[QSZZ23]{qszz23}
Lianke Qin, Zhao Song, Lichen Zhang, and Danyang Zhuo.
\newblock An online and unified algorithm for projection matrix vector
  multiplication with application to empirical risk minimization.
\newblock In {\em AISTATS}, 2023.

\bibitem[RBL{\etalchar{+}}22]{rbl+22}
Robin Rombach, Andreas Blattmann, Dominik Lorenz, Patrick Esser, and Bj{\"o}rn
  Ommer.
\newblock High-resolution image synthesis with latent diffusion models.
\newblock In {\em Proceedings of the IEEE/CVF conference on computer vision and
  pattern recognition}, pages 10684--10695, 2022.

\bibitem[RF18]{rf18}
Joseph Redmon and Ali Farhadi.
\newblock Yolov3: An incremental improvement.
\newblock {\em arXiv preprint arXiv:1804.02767}, 2018.

\bibitem[RNS{\etalchar{+}}18]{rns+18}
Alec Radford, Karthik Narasimhan, Tim Salimans, Ilya Sutskever, et~al.
\newblock Improving language understanding by generative pre-training.
\newblock 2018.

\bibitem[RSR{\etalchar{+}}20]{rsr+20}
Colin Raffel, Noam Shazeer, Adam Roberts, Katherine Lee, Sharan Narang, Michael
  Matena, Yanqi Zhou, Wei Li, and Peter~J Liu.
\newblock Exploring the limits of transfer learning with a unified text-to-text
  transformer.
\newblock {\em The Journal of Machine Learning Research}, 21(1):5485--5551,
  2020.

\bibitem[RSY{\etalchar{+}}21]{rsy+21}
V.~Ruehle, R.~Sim, Sergey Yekhanin, D.~Jones, K.~Laine, B.~Köpf, J.~Teevan,
  J.~Kleewein, and S.~Rajmohan.
\newblock Privacy preserving machine learning: Maintaining confidentiality and
  preserving trust.
\newblock 2021.

\bibitem[RWC{\etalchar{+}}19]{rwc+19}
Alec Radford, Jeffrey Wu, Rewon Child, David Luan, Dario Amodei, Ilya
  Sutskever, et~al.
\newblock Language models are unsupervised multitask learners.
\newblock {\em OpenAI blog}, 1(8):9, 2019.

\bibitem[Sag18]{s18}
Matthew Sag.
\newblock The new legal landscape for text mining and machine learning.
\newblock {\em J. Copyright Soc'y USA}, 66:291, 2018.

\bibitem[SMN{\etalchar{+}}24]{smn+24}
Zhenmei Shi, Yifei Ming, Xuan-Phi Nguyen, Yingyu Liang, and Shafiq Joty.
\newblock Discovering the gems in early layers: Accelerating long-context llms
  with 1000x input token reduction.
\newblock {\em arXiv preprint arXiv:2409.17422}, 2024.

\bibitem[{Sna}22]{s22}
{Snapchat}.
\newblock Differential privacy at snapchat, 2022.

\bibitem[Son19]{song19}
Zhao Song.
\newblock {\em Matrix theory: optimization, concentration, and algorithms}.
\newblock The University of Texas at Austin, 2019.

\bibitem[SSC{\etalchar{+}}22]{ssc+22}
Weiyan Shi, Ryan Shea, Si~Chen, Chiyuan Zhang, Ruoxi Jia, and Zhou Yu.
\newblock Just fine-tune twice: Selective differential privacy for large
  language models.
\newblock In {\em Proceedings of the 2022 Conference on Empirical Methods in
  Natural Language Processing}, pages 6327--6340, 2022.

\bibitem[SWXL24]{swxl24}
Zhenmei Shi, Junyi Wei, Zhuoyan Xu, and Yingyu Liang.
\newblock Why larger language models do in-context learning differently?
\newblock {\em arXiv preprint arXiv:2405.19592}, 2024.

\bibitem[SYYZ23]{syyz23}
Zhao Song, Xin Yang, Yuanyuan Yang, and Lichen Zhang.
\newblock Sketching meets differential privacy: Fast algorithm for dynamic
  kronecker projection maintenance.
\newblock In {\em ICML}, 2023.

\bibitem[TM22]{tm22}
Abhradeep Thakurta and Brendan McMahan.
\newblock Federated learning with formal differential privacy guarantees.
\newblock 2022.

\bibitem[TPG{\etalchar{+}}17]{tpg17}
Florian Tram{\`e}r, Nicolas Papernot, Ian Goodfellow, Dan Boneh, and Patrick
  McDaniel.
\newblock The space of transferable adversarial examples.
\newblock {\em arXiv preprint arXiv:1704.03453}, 2017.

\bibitem[UAS{\etalchar{+}}20]{uas+20}
Mohd Usama, Belal Ahmad, Enmin Song, M~Shamim Hossain, Mubarak Alrashoud, and
  Ghulam Muhammad.
\newblock Attention-based sentiment analysis using convolutional and recurrent
  neural network.
\newblock {\em Future Generation Computer Systems}, 113:571--578, 2020.

\bibitem[Vad17]{v17}
Salil Vadhan.
\newblock The complexity of differential privacy.
\newblock {\em Tutorials on the Foundations of Cryptography: Dedicated to Oded
  Goldreich}, pages 347--450, 2017.

\bibitem[Ver18]{v18}
Roman Vershynin.
\newblock {\em High-dimensional probability: An introduction with applications
  in data science}, volume~47.
\newblock Cambridge university press, 2018.

\bibitem[VKB23]{vkb23}
Nikhil Vyas, Sham Kakade, and Boaz Barak.
\newblock Provable copyright protection for generative models.
\newblock {\em arXiv preprint arXiv:2302.10870}, 2023.

\bibitem[VSP{\etalchar{+}}17]{vsp+17}
Ashish Vaswani, Noam Shazeer, Niki Parmar, Jakob Uszkoreit, Llion Jones,
  Aidan~N Gomez, {\L}ukasz Kaiser, and Illia Polosukhin.
\newblock Attention is all you need.
\newblock {\em Advances in neural information processing systems}, 30, 2017.

\bibitem[Wai19]{w19}
Martin~J Wainwright.
\newblock {\em High-dimensional statistics: A non-asymptotic viewpoint},
  volume~48.
\newblock Cambridge university press, 2019.

\bibitem[WCZ{\etalchar{+}}23]{wcz+23}
Yilin Wang, Zeyuan Chen, Liangjun Zhong, Zheng Ding, Zhizhou Sha, and Zhuowen
  Tu.
\newblock Dolfin: Diffusion layout transformers without autoencoder.
\newblock {\em arXiv preprint arXiv:2310.16305}, 2023.

\bibitem[WHHL24]{whhl24}
Dennis Wu, Jerry Yao-Chieh Hu, Teng-Yun Hsiao, and Han Liu.
\newblock Uniform memory retrieval with larger capacity for modern hopfield
  models.
\newblock In {\em Forty-first International Conference on Machine Learning
  (ICML)}, 2024.

\bibitem[WHL{\etalchar{+}}24]{whl+24}
Dennis Wu, Jerry Yao-Chieh Hu, Weijian Li, Bo-Yu Chen, and Han Liu.
\newblock {ST}anhop: Sparse tandem hopfield model for memory-enhanced time
  series prediction.
\newblock In {\em The Twelfth International Conference on Learning
  Representations (ICLR)}, 2024.

\bibitem[WMS{\etalchar{+}}24]{wms+24}
Jiayu Wang, Yifei Ming, Zhenmei Shi, Vibhav Vineet, Xin Wang, and Neel Joshi.
\newblock Is a picture worth a thousand words? delving into spatial reasoning
  for vision language models.
\newblock {\em arXiv preprint arXiv:2406.14852}, 2024.

\bibitem[WSD{\etalchar{+}}23]{wsd+23}
Zirui Wang, Zhizhou Sha, Zheng Ding, Yilin Wang, and Zhuowen Tu.
\newblock Tokencompose: Grounding diffusion with token-level supervision.
\newblock {\em arXiv preprint arXiv:2312.03626}, 2023.

\bibitem[WXZ{\etalchar{+}}24]{wxz+24}
Yilin Wang, Haiyang Xu, Xiang Zhang, Zeyuan Chen, Zhizhou Sha, Zirui Wang, and
  Zhuowen Tu.
\newblock Omnicontrolnet: Dual-stage integration for conditional image
  generation.
\newblock In {\em Proceedings of the IEEE/CVF Conference on Computer Vision and
  Pattern Recognition}, pages 7436--7448, 2024.

\bibitem[XGH{\etalchar{+}}21]{xgh+21}
Hu~Xu, Gargi Ghosh, Po-Yao Huang, Prahal Arora, Masoumeh Aminzadeh, Christoph
  Feichtenhofer, Florian Metze, and Luke Zettlemoyer.
\newblock Vlm: Task-agnostic video-language model pre-training for video
  understanding.
\newblock {\em arXiv preprint arXiv:2105.09996}, 2021.

\bibitem[XHH{\etalchar{+}}24]{xhh+24}
Chenwei Xu, Yu-Chao Huang, Jerry Yao-Chieh Hu, Weijian Li, Ammar Gilani,
  Hsi-Sheng Goan, and Han Liu.
\newblock Bishop: Bi-directional cellular learning for tabular data with
  generalized sparse modern hopfield model.
\newblock In {\em Forty-first International Conference on Machine Learning
  (ICML)}, 2024.

\bibitem[XSL24]{xsl24}
Zhuoyan Xu, Zhenmei Shi, and Yingyu Liang.
\newblock Do large language models have compositional ability? an investigation
  into limitations and scalability.
\newblock In {\em ICLR 2024 Workshop on Mathematical and Empirical
  Understanding of Foundation Models}, 2024.

\bibitem[XSW{\etalchar{+}}23]{xsw+23}
Zhuoyan Xu, Zhenmei Shi, Junyi Wei, Fangzhou Mu, Yin Li, and Yingyu Liang.
\newblock Towards few-shot adaptation of foundation models via multitask
  finetuning.
\newblock In {\em The Twelfth International Conference on Learning
  Representations}, 2023.

\bibitem[XZA{\etalchar{+}}23]{xza+23}
Zheng Xu, Yanxiang Zhang, Galen Andrew, Christopher~A Choquette-Choo, Peter
  Kairouz, H~Brendan McMahan, Jesse Rosenstock, and Yuanbo Zhang.
\newblock Federated learning of gboard language models with differential
  privacy.
\newblock {\em arXiv preprint arXiv:2305.18465}, 2023.

\bibitem[YDY{\etalchar{+}}19]{ydy+19}
Zhilin Yang, Zihang Dai, Yiming Yang, Jaime Carbonell, Russ~R Salakhutdinov,
  and Quoc~V Le.
\newblock Xlnet: Generalized autoregressive pretraining for language
  understanding.
\newblock {\em Advances in neural information processing systems}, 32, 2019.

\bibitem[YNB{\etalchar{+}}21]{ynb+21}
Da~Yu, Saurabh Naik, Arturs Backurs, Sivakanth Gopi, Huseyin~A Inan, Gautam
  Kamath, Janardhan Kulkarni, Yin~Tat Lee, Andre Manoel, Lukas Wutschitz,
  et~al.
\newblock Differentially private fine-tuning of language models.
\newblock {\em arXiv preprint arXiv:2110.06500}, 2021.

\bibitem[ZHDK23]{zhdk23}
Amir Zandieh, Insu Han, Majid Daliri, and Amin Karbasi.
\newblock Kdeformer: Accelerating transformers via kernel density estimation.
\newblock {\em arXiv preprint arXiv:2302.02451}, 2023.

\bibitem[ZHJL24]{zhjl24}
Jingyi Zhang, Jiaxing Huang, Sheng Jin, and Shijian Lu.
\newblock Vision-language models for vision tasks: A survey.
\newblock {\em IEEE Transactions on Pattern Analysis and Machine Intelligence},
  2024.

\bibitem[ZTL{\etalchar{+}}17]{ztl+17}
Jingwei Zhang, Lei Tai, Ming Liu, Joschka Boedecker, and Wolfram Burgard.
\newblock Neural slam: Learning to explore with external memory.
\newblock {\em arXiv preprint arXiv:1706.09520}, 2017.

\end{thebibliography}
\bibliographystyle{alpha}
\input{checklist}
\fi

\newpage
\onecolumn
\appendix

\ifdefined\isarxiv

\begin{center}
    \textbf{\LARGE Appendix }
\end{center}

\else

\aistatstitle{
Differentially Private Attention Computation: \\
Supplementary Materials}

{\hypersetup{linkcolor=black}
\tableofcontents
\bigbreak
\bigbreak
}

\fi

\paragraph{Roadmap.}

In Section~\ref{sec:app_preli}, we present preliminaries for the paper.
In Section~\ref{sec:error_control}, we analyze the perturbations in attention computation. 
Section~\ref{sec:gaussian_sampling_mechanism} presents the proof of the existence of differential privacy using our Gaussian sampling mechanism. In Section~\ref{sec:app_sensitivity}, we provide more Lemma about sensitivity. 

\section{PRELIMINARY}\label{sec:app_preli}

Section~\ref{sec:app_preli:notations} presents the notations that are used throughout our paper. These notations are essential for a clear and concise presentation of our work. In Section~\ref{sec:app_preli:basic_algebra}, we provide an introduction to some basic algebraic concepts that are relevant to our research. This includes fundamental mathematical operations and properties that are used in the analysis and development of our differential privacy algorithm.

\subsection{Notations}\label{sec:app_preli:notations}
For a event $C$, $\Pr[C]$ represents the probability of event $C$ occurring. $\mathbb{E}[X]$ represents the expected value (or mean) of a random variable $X$. 

We use $\chi_d^2$ to denote a Chi-squared random variable with $d$ degrees of freedom. $\mathbb{N}$ represents the set of natural numbers, which consists of all positive integers including 1, 2, 3, and so on. 

If $M$ and $N$ are symmetric matrices, we define $M \succeq N$ to mean that for all vectors $x$, the inequality $x^\top M x \geq x^\top N x$ holds.

If $M$ is a symmetric matrix of dimension $n \times n$, we define $M$ to be positive semidefinite ($M \succeq 0$) if the inequality $x^\top M x \geq 0$ holds for all vectors $x \in \R^n$.

We use the notation ${\bf 0}_n$ to denote an $n$-dimensional vector whose entries are all zero, and ${\bf 1}_n$ to denote an $n$-dimensional vector whose entries are all one. The symbol $I_n$ represents the $n \times n$ identity matrix, which is a square matrix with ones on the main diagonal and zeros elsewhere.

Let $x$ be an arbitrary vector in $\R^n$. We define $\exp(x) \in \R^n$ as a vector whose $i$-th entry $\exp(x)_i$ is equal to $\exp(x_i)$, where $\exp(\cdot)$ denotes the exponential function. We use $\langle x,y \rangle$ to denote $\sum_{i=1}^n x_i y_i$. 

For any matrix $A$, we use $\| A \|$ to denote the spectral norm of $A$, i.e., $\|A\| = \max_{\|x\|_2 = 1} \|Ax\|_2$, $\| A \|_F$ to denote its Frobenius norm and $\| A \|_\infty$ to denote the infinity norm. $A_{i,j}$ represents the element in the $i$-th row and $j$-th column of matrix $A$. $\det(A)$ represents the determinant of matrix $A$. For a square and symmetric matrix $A \in \R^{n \times n}$, we say $A$ positive semi-definite ($A \succeq 0$) if for all vectors $x \in \R^n$, we have $x^\top A x \geq 0$.

We denote the inverse of a matrix $M$ as $M^{-1}$ and its transpose as $M^\top$. We refer to $\lambda_i$ as the $i$-th eigenvalue of $N$.

$\mathbb{S}^n_+$ denotes the set of $n \times n$ positive semidefinite (PSD) matrices.

\subsection{Basic Algebra}\label{sec:app_preli:basic_algebra}
In this section, we offer an introduction to fundamental algebraic concepts.

\begin{fact}\label{fac:norm}
We have
\begin{itemize}
    \item Part 1. Let $A\in \R^{n \times n}$, then we have $\| A \|_{F} \leq \sqrt{n} \| A \|$.
    \item Part 2. Let $A \in \R^{n \times n}$, then we have $\| A \| \leq \| A \|_F$
    \item Part 3. For two vectors $a,b \in \R^n$, then we have $\| a b^\top \| \leq \| a \|_2 \cdot \| b \|_2$
\end{itemize}
\end{fact}

\begin{fact}\label{fac:exp_cosh}
We have
\begin{itemize}
    \item Part 1. $\cosh(x) = \sum_{i=0}^{\infty} (1/(2i)!) \cdot x^{2i} $. 
    \item Part 2. $\exp(x) = \sum_{i=0}^{\infty} (1/(i!)) \cdot x^i$.
    \item Part 3.  We have $|\exp(x) - 1| \leq |x| + x^2$, $\forall x \in (-0.1,0.1)$.
    \item Part 4. $| \exp(x) - \exp(y) | \leq \exp(x) \cdot ( |x-y| + |x-y|^2) $ for $|x-y| \leq 0.1$.
    \item Part 5. We have $| \cosh(x) - 1| \leq x^2$, $\forall x \in (-0.1,0.1)$.
    \item Part 6. $|\cosh(x) - \cosh(y)| \leq \cosh(x) \cdot |x-y|^2$ for $|x-y| \leq 0.1$.
\end{itemize}
\end{fact}

\section{ERROR CONTROL FROM LOGIT MATRIX TO ATTENTION MATRIX}\label{sec:error_control}

In Section~\ref{sec:error_control:perturb_psd_matrix}, we discuss the perturbation of positive semi-definite (psd) matrices, which is a crucial step in ensuring the differential privacy of our algorithm. 
Section~\ref{sec:error_control:perturb_diagonal_normalization_matrix} focuses on the perturbation of diagonal normalization matrices, which is another important aspect of our error control approach. In 
Section~\ref{sec:error_control:error_attenrion_matrix}, we analyze the error in the attention matrix computation that arises from these perturbations. 
Finally, in Section~\ref{sec:error_control:main_result}, we present the main result of Section~\ref{sec:error_control}, which summarizes the effectiveness of our error control mechanisms in achieving differential privacy for the computation of the attention matrix.

\subsection{Perturb PSD Matrix}\label{sec:error_control:perturb_psd_matrix}

In Section~\ref{sec:error_control:perturb_psd_matrix}, we discuss the perturbation of positive semi-definite (psd) matrices. This is a crucial step in ensuring the differential privacy of our algorithm.

\begin{lemma}[Lemma 3.1 in \cite{dms23}]\label{lem:perturb_psd}
We denote $A \in \R^{n \times n}$ and $B \in \R^{n \times n}$ as psd matrices.

If all of the following requirements are met
\begin{itemize}
    \item {\bf Requirement 1.} We have $-r \leq A_{i,j} \leq r$, $\forall (i,j) \in [n] \times [n]$.
    \item {\bf Requirement 2.} $(1-\epsilon)B \preceq A \preceq (1+\epsilon) B$;
\end{itemize}
Then, it follows that
\begin{align*}
    B_{i,j} \in [ -(1+\epsilon)r, (1+\epsilon) r ].
\end{align*}
\end{lemma}

\begin{lemma}[A general version of Lemma 3.2 in \cite{dms23}]
\label{lem:perturb_exp}
    If all of the following requirements are met
    \begin{itemize}
        \item {\bf Requirement 1.} $A_{i,j} \in [-r, r]$.
        \item {\bf Requirement 2.} $B_{i,j} \in [-(1+\epsilon)r, (1+\epsilon) r ]$.
        \item {\bf Requirement 3.} $r \in (0,0.1)$, $\epsilon \in (0,0.1)$.
        \item {\bf Requirement 4.} Let $f(z) \in \{ \exp(z), \cosh(z) \}$.
    \end{itemize}
    It follows that
    \begin{itemize}
    \item {\bf Part 1.}
    \begin{align*}
        | f(A_{i,j}) - f(B_{i,j}) | \leq f(A_{i,j}) \cdot (2 + 2\epsilon + 4r ) \cdot r ~~\forall i,j \in [n] \times [n].
    \end{align*}
    \item {\bf Part 2.}
    \begin{align*}
        | f( A_{i,j} ) - f( B_{i,j} ) | \leq f( B_{i,j} ) \cdot (2 + 2\epsilon + 4r) \cdot r ~~\forall i,j \in [n] \times [n].
    \end{align*}
    \end{itemize}
\end{lemma}
\begin{proof}
According to {\bf Requirement 1.}, {\bf Requirement 2.} and {\bf Requirement 3.}, we have
\begin{align}\label{eq:bound_A_i_j_minus_B_i_j}
    | A_{i,j} - B_{i,j} | \leq (2+\epsilon) r .
\end{align}

\paragraph{Proof of Part 1.}
    It follows that
    \begin{align*}
    | f(A_{i,j}) - f(B_{i,j}) |
    \leq & ~ f(A_{i,j}) \cdot ( |A_{i,j} - B_{i,j}| + |A_{i,j} - B_{i,j}|^2 )\\
    \leq & ~ f(A_{i,j}) \cdot |A_{i,j} - B_{i,j}| \cdot ( 1 + |A_{i,j} - B_{i,j}| ) \\
    \leq & ~ f(A_{i,j}) \cdot |A_{i,j} - B_{i,j}| \cdot (1 + (2+\epsilon) r ) \\
    \leq & ~ f(A_{i,j}) \cdot (2+\epsilon) r \cdot (1+ (2+\epsilon)r) \\
    = & ~ f(A_{i,j}) \cdot (2 + \epsilon + (2+ \epsilon)^2 r ) r \\
    \leq & ~ f(A_{i,j}) \cdot (2 + 2\epsilon + 4r ) r
    \end{align*}
    where the 1st step is the result of Fact~\ref{fac:exp_cosh}, the 2nd step follows from straightforward algebraic manipulations, the 3rd step is a consequence of Eq.\eqref{eq:bound_A_i_j_minus_B_i_j}, the 4th step is a consequence of Eq.\eqref{eq:bound_A_i_j_minus_B_i_j}, the 5th step follows from algebraic manipulations, and the 6th step is a result of satisfying {\bf Requirement 3} in the Lemma statement.
    
\paragraph{Proof of Part 2.}

Similarly, we can prove it.

\end{proof}

\subsection{Error Control for Normalization}\label{sec:error_control:perturb_diagonal_normalization_matrix}

This section focuses on the perturbation of diagonal normalization matrices, which is another important aspect of our error control approach.

\begin{lemma}[Error Control for Normalization, A general version Lemma 3.3 in \cite{dms23}]\label{lem:perturb_D}
If the following condition holds
\begin{itemize}
    \item {\bf Requirement 1.} We define $f$ as Definition~\ref{def:f}.
    \item {\bf Requirement 2.} We define $\D$ as Definition~\ref{def:D}.
    \item {\bf Requirement 3.} $\forall (i,j) \in [n]\times [n]$, we have $| f(A_{i,j}) - f(B_{i,j}) | \leq f(A_{i,j}) \cdot c_0 r$.
    \item {\bf Requirement 4.} $\forall (i,j) \in [n]\times [n]$, we have $ f(A_{i,j}) - f(B_{i,j}) | \leq f(B_{i,j}) \cdot c_0 r$.
\end{itemize}
Then, it follows that,
\begin{itemize}
\item {\bf Part 1.}
\begin{align*}
    | \D(A)_{i,i} - \D(B)_{i,i} | \leq \D(A)_{i,i} \cdot c_0 r ~~\forall i \in [n]
\end{align*}
\item {\bf Part 2.} 
\begin{align*}
    | \D(A)_{i,i} - \D(B)_{i,i} | \leq \D(B)_{i,i} \cdot c_0 r ~~\forall i \in [n]
\end{align*}
\end{itemize}
\end{lemma}
\begin{proof}

{\bf Proof of Part 1.}
From the above conditions in the lemma statement, it follows that
\begin{align*}
| \D(A)_{i,i} - \D(B)_{i,i}|
= & ~ | ( f( A_{i,*} ) - f(B_{i,*}) ) \cdot {\bf 1}_n |  \\
= & ~ | \sum_{j=1}^n ( f(A_{i,j}) - f(B_{i,j}) ) | \\
\leq & ~ \sum_{j=1}^n | f( A_{i,j} ) - f( B_{i,j} ) | \\
\leq & ~ \sum_{j=1}^n f(A_{i,j}) \cdot c_0 r \\
= & ~ f (A_{i,*}) {\bf 1}_n \cdot c_0 r \\
= & ~ \D(A)_{i,i} \cdot c_0 r
\end{align*} 
where the 1st step follows from algebraic manipulations, the 2nd step is due to algebraic manipulations, the 3rd step is the result of triangle inequality, the 4th step is based on {\bf Requirement 2} in Lemma statement, the 5th step comes from algebraic manipulations and the last step is the result of algebraic manipulations. 

{\bf Proof of Part 2.}

The proof is similar to Part 1. So we omit the details here.

\end{proof}

\subsection{Error of Attention Matrix}\label{sec:error_control:error_attenrion_matrix}

In this section, we analyze the error in the attention matrix computation that arises from the perturbations of psd and diagonal normalization matrices. 

\begin{lemma}[A general version of Lemma 3.4 in \cite{dms23}]
\label{lem:perturb_attention}
Let $c_1 > 0$ and $c_2 >0$.
    If all of the following requirements are met
    \begin{itemize}
        \item {\bf Requirement 1.} We define $f$  as Definition~\ref{def:f}.
        \item {\bf Requirement 2.} We define $\D$  as Definition~\ref{def:D}.
        \item {\bf Requirement 3.} 
        \begin{align*}
            | \D(A)_{i,i} - \D(B)_{i,i} | \leq c_1 \cdot r \cdot \min \{ \D(A)_{i,i}, \D(B)_{i,i} \}~~ \forall i \in [n],
        \end{align*}
    \item {\bf Requirement 4.} 
        \begin{align*}
            | f( A_{i,j} ) - f( B_{i,j} ) | \leq c_2 \cdot r \cdot \min\{ f(A_{i,j} ), f( B_{i,j} ) \} ~~ \forall  i,j \in [n] \times [n]
        \end{align*}
    \end{itemize}
    It follows that
    \begin{align*}
        \|\D( A )^{-1} f(A) - \D( B )^{-1} f(B)\|_\infty \le (c_1+c_2) \cdot r.
    \end{align*}
\end{lemma}

\begin{proof}
    We first decompose the difference into
    \begin{align*}
        & ~ \|\D( A )^{-1} f(A) - \D( B )^{-1} f(B) \|_\infty \\
    \le & ~ \|\D( A )^{-1}f(A) - \D( B )^{-1} f(B) \|_\infty + \|\D( B )^{-1} f(B) - \D( B )^{-1} f(B) \|_\infty  \\
    = & ~ Z_1 + Z_2
    \end{align*}
where last step is obtained by
\begin{align*}
    Z_1 := \|\D( B )^{-1} f(B) - \D( B )^{-1} f(B)\|_\infty,
\end{align*}
and
\begin{align*}
    Z_2 := \|\D( A )^{-1} f(A) - \D( B )^{-1} f(B)\|_\infty.
\end{align*}

We will present the proof in two parts.

    \paragraph{The first term.}
    $\forall (i, j) \in [n] \times [n]$, it follows that
    \begin{align*}
     Z_1 
     = & ~ |(\D( A )^{-1} f(A) - \D( B )^{-1} f(B) )_{i, j}| \\
    = & ~ | \D( A )^{-1}_{i,i}\cdot ( f(A)_{i,j} - f(B)_{i, j})| \\
    \leq & ~ \D( A )^{-1}_{i,i} \cdot | f(A)_{i,j} - f(B)_{i, j})| \\
    \leq & ~ \D( A )^{-1}_{i,i} \cdot c_2 \cdot r \cdot f(A)_{i,j}  \\
    \le & ~ c_2 r \cdot (\D( A )^{-1} f(A) )_{i,j} \\
    \leq & ~ c_2 r,
    \end{align*}
    where the 1st step comes from definition, the 2nd step is the result of algebraic manipulations, the 3rd step comes from triangle inequality, the 4th step is based on  {\bf Requirement 4} in the lemma statement, the 5th step is the result of algebraic manipulations, and the last step is according to the definition of $\D$.  
    
    \paragraph{The second term.}
    $\forall (i, j) \in [n] \times [n]$, it follows that
    \begin{align*}
    Z_2 = & ~ |(\D( B )^{-1} f(B) - \D( B )^{-1} f(B) )_{i, j}| \\
    =  & ~ | ( \D( A )^{-1}_{i,i} - \D( A )^{-1}_{i,i}) f(B)_{i,j}| \\
    =  & ~ | \frac{\D(A)_{i,i}-\D(B)_{i,i}}{\D(A)_{i,i}\D(B)_{i,i}} f(B)_{i,j}| \\
    \leq & ~ | \frac{\D(A)_{i,i}-\D(B)_{i,i}}{\D(A)_{i,i}\D(B)_{i,i}} | \cdot | f(B)_{i,j} | \\
     \leq & ~ | \frac{c_{1} r \D(A)_{i,i}}{\D(A)_{i,i}\D(B)_{i,i}} | \cdot |   f(B)_{i,j} | \\
    = & ~ c_1 r \cdot |\D(B)^{-1}_{i,i} | \cdot |  f(B)_{i,j}|
    \end{align*}
    where the 1st step based on definition, the 2nd steps follow from algebraic manipulations, the 3rd step is the result of algebraic manipulations, the 4th step is due to triangle inequality, the 5th step is due to  {\bf Requirement 3} in the lemma statement, the last step is due to algebraic manipulations.

    Then we have
    \begin{align*}
    Z_2 = & ~ c_1 r \cdot |\D(B)^{-1}_{i,i} | \cdot |  f(B)_{i,j}| \\
    = & ~ c_1 r \cdot |\D(B)^{-1}_{i,i} f(B)_{i,j}| \\
    = & ~ c_1 r \cdot (\D(B)^{-1} f(B) )_{i,j} \\
    \leq & ~ c_1 r
    \end{align*}
    where the 1st step is the result of the above equation, the 2nd step is due to all the entries are positive, the 3rd step is due to algebraic manipulations and the last step is due to definition of $\D$.

    Based on the above deduction, it follows that
     \begin{align*}
        \|\D( A )^{-1} f(A) - \D( B )^{-1} f(B) \|_\infty 
        \leq & ~ Z_1 + Z_2 \\
        \leq & ~ ( c_1 +c_2 ) r.
    \end{align*}
    
    Thus we complete the proof. 
\end{proof}

\subsection{Error Control}\label{sec:error_control:main_result}
The main result of Section~\ref{sec:error_control} is presented in this section.

\begin{theorem}[Formal version of Theorem~\ref{thm:perturb_psd}]
If all of the following requirements are met
\begin{itemize}
    \item Let $\epsilon \in (0,0.1)$
    \item Let $r \in (0,0.1)$
    \item $\| A \|_{\infty} \leq r$
    \item $(1-\epsilon) B \preceq A \preceq (1+\epsilon) B$
    \item We define $\D$  Definition~\ref{def:D}.
    \item We define $f$ as Definition~\ref{def:f}.
\end{itemize}
It follows that
\begin{align*}
    \| \D( A )^{-1} f(A) - \D( B )^{-1} f(B) \|_{\infty} \leq 4 \cdot (1 + \epsilon + 2 r) \cdot r
\end{align*}
\end{theorem}

\begin{proof}[Proof of Theorem~\ref{thm:perturb_psd}]
By Lemma~\ref{lem:perturb_psd} and $(1-\epsilon) B \preceq A \preceq (1+\epsilon) B$, we have 
\begin{align}\label{eq:perturb_psd}
    B_{i,j} \in [ -(1+\epsilon)r, (1+\epsilon) r ].
\end{align}. 

By Lemma~\ref{lem:perturb_exp} and Eq.~\eqref{eq:perturb_psd}, it follows that
\begin{itemize}
    \item {\bf Part 1.}
    \begin{align*}
        | f(A_{i,j}) - f(B_{i,j}) | \leq f(A_{i,j}) \cdot (2 + 2\epsilon + 4r ) \cdot r ~~ \forall (i,j) \in [n] \times [n].
    \end{align*}
    \item {\bf Part 2.} 
    \begin{align*}
        | f( A_{i,j} ) - f( B_{i,j} ) | \leq f( B_{i,j} ) \cdot (2 + 2\epsilon + 4r) \cdot r  ~~ \forall (i,j) \in [n] \times [n].
    \end{align*}
    \end{itemize}
According to the discussion above and using Lemma~\ref{lem:perturb_D}, we have 
\begin{itemize}
\item {\bf Part 1.}
\begin{align*}
    | \D(A)_{i,i} - \D(B)_{i,i} | \leq \D(A)_{i,i} \cdot c_0 r ~~ \forall i \in [n]
\end{align*}
\item {\bf Part 2.} 
\begin{align*}
    | \D(A)_{i,i} - \D(B)_{i,i} | \leq \D(B)_{i,i} \cdot c_0 r ~~ \forall i \in [n]
\end{align*}
\end{itemize}
And then by using Lemma~\ref{lem:perturb_attention}, $c_1 =(2 + 2\epsilon + 4r)$ and $c_2 = (2 + 2\epsilon + 4r)$, we have 
\begin{align*}
    \| \D( A )^{-1} f(A) - \D( B )^{-1} f(B) \|_{\infty} \leq 4 \cdot (1 + \epsilon + 2 r) \cdot r
\end{align*}
\end{proof}

\section{ANALYSIS OF GAUSSIAN SAMPLING MECHANISM}\label{sec:gaussian_sampling_mechanism}

We denote the output of our privacy algorithm as $Z$. 
In Section~\ref{sec:gaussian_sampling_mechanism:computation_tools}, we present the computation tools that we use to implement our approach. 
In Section~\ref{sec:gaussian_sampling_mechanism:spectral_decomposition}, we perform spectral decomposition of $A :=  {\cal M}({\cal Y})^{1/2} {\cal M}({\cal Y'})^{-1} {\cal M}({\cal Y})^{1/2}$ and derive some important conclusions from it. 
Then, in Section~\ref{sec:gaussian_sampling_mechanism:tranformation}, we transform $Z$ into a format that is based on the spectral decomposition of $A$. 
In Section~\ref{sec:gaussian_sampling_mechanism:upper_bound}, We present the upper bound of $\E[Z]$, which is useful in the following section.
In Section~\ref{sec:gaussian_sampling_mechanism:sub-exponential}, we demonstrate that $Z$ is sub-exponential, which allows us to control the upper bound of $\Pr[Z \geq \epsilon]$ where $\epsilon \in (0,1)$. 
Finally, we present our main result in Section~\ref{sec:gaussian_sampling_mechanism:main_result}, which is that our Algorithm~\ref{alg:the_gaussian_sampling_mechanism} is differential privacy.

\subsection{Computation Tools}\label{sec:gaussian_sampling_mechanism:computation_tools}

This section is dedicated to presenting the computational tools that we use to implement our approach. 

\begin{definition}\label{def:A_B_C}
    We define $\Sigma_1,\Sigma_2$ as Definition~\ref{def:sigma}. Let us define
    \begin{itemize}
        \item $ A := \Sigma_1^{1/2}\Sigma_2^{-1}\Sigma_1^{1/2}$
        \item $B := \Sigma_2^{1/2}\Sigma_1^{-1}\Sigma_2^{-1/2}$
        \item $C:=\Sigma_1^{-1/2}\Sigma_2^{1/2}$
    \end{itemize}
\end{definition}

\begin{lemma}\label{lem:computaion_A_B}
Let $A,B$ and $C$ be defined as Definition~\ref{def:A_B_C}. Then we have 
\begin{itemize}
        \item {\bf Part 1.} $A^{-1}=CC^\top$.
        \item {\bf Part 2.} $B =C^\top C$.
        \item {\bf Part 3.} $A^{-1},B$ have the same eigenvalue.
\end{itemize}
\end{lemma}

\begin{proof}
Note that $\Sigma_1$ and $\Sigma_2$ are  symmetric, we can easily have the proof as follows.
\paragraph{Proof of Part 1.}
    \begin{align}\label{eq:computation_A}
        A^{-1} & ~ =(\Sigma_1^{1/2}\Sigma_2^{-1}\Sigma_1^{1/2})^{-1} \notag \\
        & ~ = (\Sigma_1^{1/2}\Sigma_2^{-1/2}\Sigma_2^{-1/2}\Sigma_1^{1/2})^{-1} \notag \\
        & ~ = (\Sigma_2^{-1/2}\Sigma_1^{1/2})^{-1}(\Sigma_1^{1/2}\Sigma_2^{-1/2})^{-1} \notag \\
        & ~ = (\Sigma_1^{1/2}\Sigma_2^{-1/2})(\Sigma_2^{-1/2}\Sigma_1^{1/2}) \notag \\
        & ~ = CC^\top
    \end{align}
    
\paragraph{Proof of Part 2.}
\begin{align}\label{eq:computation_B}
    B & ~ = \Sigma_2^{-1/2}\Sigma_1\Sigma_2^{-1/2} \notag \\
    & ~ = (\Sigma_2^{-1/2}\Sigma_1^{1/2})(\Sigma_1^{1/2}\Sigma_2^{-1/2}) \notag \\
     & ~ = C^{T} C
\end{align}
\paragraph{Proof of Part 3.}
It simply follows from Eq.\eqref{eq:computation_A} and Eq.\eqref{eq:computation_B}.
\end{proof}
\subsection{Spectral Decomposition}\label{sec:gaussian_sampling_mechanism:spectral_decomposition}
This section is focused on the spectral decomposition of $A$, which we perform to gain insights into its properties. By analyzing the spectral decomposition, we are able to draw important conclusions about $A$ that are relevant to our approach. 
\begin{lemma}\label{lem:spectral_decomposition_A}
    If all of the following requirements are met
    \begin{itemize}
        \item {\bf Requirement 1.} We define $A$ as Definition~\ref{def:A_B_C}.
        \item {\bf Requirement 2.} Let $\lambda_1 \cdots \lambda_n$ be eigenvalues of $A$.
        \item {\bf Requirement 3.} Let $A = \sum_{j=1}^n \lambda_j v_j v_j^\top$ be spectral decomposition for $A$.
        \item  {\bf Requirement 4.} Let $\Delta$ be denoted as Definition~\ref{def:delta}.
        \item  {\bf Requirement 5.} Let $M,{\cal M}$ be denoted as Definition~\ref{def:m} and $M \leq \Delta$.
    \end{itemize}
        We have
    \begin{itemize}
        \item $\sum_{j=1}^n(\lambda_j-1)^2 \leq  \Delta^2$.
        \item $\sum_{j=1}^n(1-\frac{1}{\lambda_j})^2 \leq \Delta^2$.
    \end{itemize}
\end{lemma}
\begin{proof}
     we have
    \begin{align*}
        \sum_{j=1}^n(\lambda_j-1)^2
        = & ~ \|A - I\|_F^2 \notag\\
        \leq & ~ \Delta^2
    \end{align*}
    where the 1st step is based on {\bf Requirement 3} in the lemma statement  and the last step is due to {\bf Requirement 5} in lemma statement.
    
    Similarly, we have
    \begin{align*}
        \sum_{j=1}^n(1-\frac{1}{\lambda_j})^2
        = & ~ \|I-A^{-1}\|_F^2 \notag\\
        = & ~ \|I-B\|_F^2 \notag\\
        \leq & ~ \Delta^2
    \end{align*}
    where the 1st step is due to {\bf Requirement 3} in the lemma statement, the 2nd step follows from swapping the roles of $\mathcal{Y},\mathcal{Y^{'}}$ and the last step is due to Lemma~\ref{lem:computaion_A_B}.
\end{proof}
\subsection{The Transformation for Output}\label{sec:gaussian_sampling_mechanism:tranformation}
In Section~\ref{sec:gaussian_sampling_mechanism:tranformation}, we describe the process of transforming the output $Z$ of our privacy algorithm into a format that is based on the spectral decomposition of $A$.
\begin{lemma}\label{lem:transformation_for_z}

If all of the following requirements are met
\begin{itemize}
    \item {\bf Requirement 1.}We define $Z$ and $h_{i,j}$  as Definition~\ref{def:z}.
    \item {\bf Requirement 2.} Let $A$ be denoted as Definition~\ref{def:A_B_C}.
    \item {\bf Requirement 3.} Let $\lambda_1,\cdots,\lambda_n$ demote the eigenvalue of $A$. 
\end{itemize}
Then we have
\begin{align*}
    Z = \frac{1}{2} \sum_{i=1}^k \sum_{j=1}^n \bigg( (\lambda_j - 1)h_{i,j}^2-\log(\lambda_j) \bigg)
\end{align*}
    
\end{lemma}
\begin{proof}
    The privacy loss random variable $Z$ can be expressed as follows:
    \begin{align*}
        Z = & ~ \sum_{i=1}^k \log \bigg( \frac{\det(\Sigma_1)^{-\frac{1}{2}}\exp(-\frac{1}{2}g_i^\top \Sigma_1^{-1} g_i)}{\det(\Sigma_2)^{-\frac{1}{2}}\exp(-\frac{1}{2}g_i^\top \Sigma_2^{-1} g_i)} \bigg) \\
        = & ~ \sum_{i=1}^k \bigg( \frac{1}{2}g_i^\top (\Sigma_2^{-1}-\Sigma_1^{-1})g_i-\frac{1}{2}\log\bigg( \frac{\det(\Sigma_1)}{\det(\Sigma_2)} \bigg) \bigg) \\
        = & ~ \frac{1}{2} \sum_{i=1}^k \bigg(\bigg( \Sigma_1^{-1/2}g_i \bigg)^\top (A-I)\bigg( \Sigma_1^{-1/2}g_i \bigg)-\log\det(A) \bigg) \\
        = & ~ \frac{1}{2} \sum_{i=1}^k \sum_{j=1}^n \bigg( (\lambda_j - 1)h_{i,j}^2-\log(\lambda_j) \bigg)
    \end{align*}
    where the 1st step is based on {\bf Requirement 1} in the lemma statement, the 2nd step follows from 
    rearranging the terms, the 3rd step is based on {\bf Requirement 2} in the lemma statement, and the last step is by taking the spectral decomposition of $A$. 
\end{proof}

\subsection{The Upper Bound for Expectation}\label{sec:gaussian_sampling_mechanism:upper_bound}
In Section~\ref{sec:gaussian_sampling_mechanism:upper_bound}, we provide an upper bound on the expected value of $Z$, which is a useful result for the subsequent section.
\begin{lemma}\label{lem:upper_bound_e_z}
If all of the following requirements are met
\begin{itemize}
    \item  {\bf Requirement 1} We define $Z$ as Definition~\ref{def:z}.
    \item  {\bf Requirement 2} Let $\epsilon \in (0,1)$ and $k \in \mathbb{N}$.
    \item {\bf Requirement 3.} Let $A$ be denoted as Definition~\ref{def:A_B_C}.
    \item {\bf Requirement 4.} Let $\lambda_1,\cdots,\lambda_n$ denote the eigenvalue of $A$. 
     \item {\bf Requirement 5.} Let $\Delta$ be denoted as Definition~\ref{def:delta}.
    \item {\bf Requirement 6.} Let $M,{\cal M}$ be denoted as Definition~\ref{def:m} and $M \leq \Delta$.
\end{itemize}
   we have 
    \begin{align*}
        \E[Z] \leq \frac{\epsilon}{2}
    \end{align*}
\end{lemma}
\begin{proof}
    \begin{align*}
        \E[Z] = & ~ \frac{k}{2} \sum_{j=1}^n (\lambda_j - 1 - \log(\lambda_j)) \\
        \leq & ~ \frac{k}{2} \sum_{j=1}^n (\lambda_j - 2 + \frac{1}{\lambda_j}) \\
        = & ~ \frac{k}{2} \sum_{j=1}^n (\lambda_j - 1)(1 - \frac{1}{\lambda_j}) \\
        \leq & ~ \|A-I\|_F \cdot \|I-A^{-1}\|_F \\
        \leq & ~ \frac{k}{2} \Delta^2 \notag\\
        \leq & ~ \frac{\epsilon}{2}
    \end{align*}
    where the 1st step follows from linearity of expectation and Lemma~\ref{lem:transformation_for_z}, the 2nd step is the result of $\lambda_j > 0$ and $\log(x) > 1 - \frac{1}{x}$ for $x > 0$, the 3rd step follows from simple factorization, the fourth step follows from Cauchy-Schwarz, the fifth step follows from Lemma~\ref{lem:spectral_decomposition_A} and {\bf Requirement~6} in the lemma statement, and the last step follows from $\Delta < \frac{\epsilon}{\sqrt{k}}$ and $\epsilon < 1$.
\end{proof}
\subsection{Sub-Exponential}\label{sec:gaussian_sampling_mechanism:sub-exponential}

In Section~\ref{sec:gaussian_sampling_mechanism:sub-exponential}, evidence is provided that supports the claim that $Z$ is sub-exponential. This is significant because it enables us to limit the maximum probability of the event ${Z \geq \epsilon}$, which is crucial in ensuring differential privacy.
\begin{lemma}\label{lem:upper_bound_pr_z}

    If all of the following requirements are met
    \begin{itemize}
        \item  {\bf Requirement 1.}We define $Z$ as Definition~\ref{def:z}.
        \item  {\bf Requirement 2.} Let $\epsilon \in (0,1)$ and $\delta \in (0,1)$.
        \item {\bf Requirement 3.} Let $\Delta$ be denoted as Definition~\ref{def:delta} and $\Delta < 1$.
        \item {\bf Requirement 4.} Let $M,{\cal M}$ be denoted as Definition~\ref{def:m} and $M \leq \Delta$.
        \item {\bf Requirement 5.} $k \in \mathbb{N}$.
    \end{itemize}
    we have
    \begin{align*}
        \Pr[Z >\epsilon]\leq \delta
    \end{align*}
\end{lemma}
\begin{proof}
First, we will prove $Z$ is sub-exponential.
\paragraph{Proof of Sub Exponential}
Let $A$ be dented as Definition~\ref{def:A_B_C} and $h_{i,j}$ be denoted as Definition~\ref{def:z}.

Since $h_{i,j}\sim \chi_1^2$, Lemma~\ref{lem:sub_exponential_parameters_independent} and Lemma~\ref{lem:sub_exponential_parameters}, we can say $Z$ is sub-exponential with 
\begin{itemize}
    \item $\nu = \sqrt{k}\|I - A\|_F$
    \item $\alpha = 2\| I - A\|_F$
\end{itemize}

By Lemma~\ref{lem:spectral_decomposition_A}, we have
\begin{itemize}
    \item $\nu=\sqrt{k}\|A-I\|_F \leq \sqrt{k}\Delta$
    \item $\alpha = 2\|A-I\|_F \leq 2 \Delta$
\end{itemize}
\paragraph{Proof of Upper Bound for $\E[Z]$.}

Under {\bf Requirement 3} and {\bf Requirement 4}, by using Lemma~\ref{lem:upper_bound_e_z}, we have
\begin{align}\label{eq:mean_z}
    \E[Z] \leq \epsilon/2
\end{align}
\paragraph{Proof of Upper Bound}

By using Lemma~\ref{lem:sub_exponential_tail_bound} (sub-exponential tail bound), 
we have
\begin{align*}
        \Pr[Z > \epsilon] 
        < & ~ \Pr[Z-\E[Z]> \epsilon/2] \\ 
        \leq & ~ \max\bigg\{ \exp(-\frac{(\epsilon/2)^2}{2\nu^2}), \exp( -\frac{\epsilon/2}{2\alpha}) \bigg\} \\ 
        \leq & ~ \delta
\end{align*}
    where the 1st step is the reuslt of Eq.~\eqref{eq:mean_z}, the 2nd step is the reuslt of Lemma~\ref{lem:sub_exponential_tail_bound}, and the last step follows from { \bf Requirement 3} in the lemma statement. 
\end{proof}

\subsection{Analysis of Gaussian Sampling}\label{sec:gaussian_sampling_mechanism:main_result}
This section contains our main result in Section~\ref{sec:gaussian_sampling_mechanism}, which we present as follows. The following theorem statement can be viewed as a variation of Theorem 5.1 in \cite{akt+22}.

\begin{theorem}[Formal version of Theorem~\ref{thm:analysis_of_the_aussian_sampling_mechanism:informal}, Analysis of the Gaussian Sampling Mechanism ]\label{thm:analysis_of_the_aussian_sampling_mechanism}
    If all of the following requirements are met
    \begin{itemize}
        \item {\bf Requirement 1.} Let $\epsilon \in (0,1)$ and $\delta \in (0,1)$.
        \item {\bf Requirement 2.} $k \in \mathbb{N}$.
        \item {\bf Requirement 3.} Let $\Delta$ be denoted as Definition~\ref{def:delta} and $ \Delta < 1$.
        
        \item  {\bf Requirement 4.} Let $M,{\cal M}$ be denoted as Definition~\ref{def:m} and $M \leq \Delta$.
        \item {\bf Requirement 5.} An input $\Sigma = \mathcal{M}(\mathcal{Y})$.
        \item {\bf Requirement 6.} $\rho = O( \sqrt{ ( n^2+\log(1/\gamma) )  / k }+ ( n^2+\log(1/\gamma) ) /{k} )$.
        \end{itemize}
    Then, there exists an algorithm~\ref{alg:the_gaussian_sampling_mechanism} such that
    \begin{itemize}
        \item Part 1. Algorithm~\ref{alg:the_gaussian_sampling_mechanism} is $(\epsilon,\delta)$-DP (with respect to the original dataset $\mathcal{Y}$).
        \item Part 2. outputs $\hat{\Sigma} \in \mathbb{S}_+^n$ such that with probabilities at least $1-\gamma$,
        \begin{align*}
            \| \Sigma^{-1/2} \wh{\Sigma} \Sigma^{-1/2}-I_n \|_F \leq \rho
        \end{align*}   
        \item Part 3. 
        \begin{align*}
           (1-\rho) \Sigma \preceq \wh{\Sigma} \preceq (1+\rho)  \Sigma
        \end{align*}  
    \end{itemize}
\end{theorem}

\begin{proof}
We denote $Z$ as Definition~\ref{def:z} which is as the output of algorithm~\ref{alg:the_gaussian_sampling_mechanism}. The utility guarantee is immediately implied by Theorem~\ref{thm:empirical_covariance_estimator_for_gaussian}. 
We then focus on the proof of privacy. By Lemma~\ref{lem:upper_bound_pr_z}, we have 
\begin{align} \label{eq:upper_bound_pr_z}
    \Pr[Z > \epsilon] \leq \delta
\end{align}
    
And then by Theorem~\ref{thm:epsilon_delta_DP} and Eq.~\eqref{eq:upper_bound_pr_z}, Algorithm~\ref{alg:the_gaussian_sampling_mechanism} is proved as $(\epsilon,\delta)$-differential private.

{\bf Proof of Part 3.}
\begin{align*}
    \| \Sigma^{-1/2} \wh{\Sigma} \Sigma^{-1/2}- I_n \| 
    \leq & ~\| \Sigma^{-1/2} \wh{\Sigma} \Sigma^{-1/2}-I_n \|_F \\
    \leq & ~ \rho
\end{align*} 
Thus,
\begin{align*}
   (1-\rho) I_n \preceq \Sigma^{-1/2} \wh{\Sigma} \Sigma^{-1/2} \preceq (1+\rho) I_n
\end{align*}
which is equivalent to
\begin{align*}
   (1-\rho) \Sigma \preceq  \wh{\Sigma}   \preceq (1+\rho) \Sigma
\end{align*}

\end{proof}

\section{MORE SENSITIVITY LEMMA}\label{sec:app_sensitivity}

In this section, we provide more lemmas related to sensitivity.

\begin{lemma}[Formal version of Lemma~\ref{lem:sensitivity}]\label{lem:sensitivity_formal}
   If  $X \in \R^{n \times d}$ and $\wt{X} \in \R^{n \times d}$ are neighboring dataset (see Definition~\ref{def:dataset} and Definition~\ref{def:neighboring}), then
   $
         (1- 2 \alpha \beta/\eta)  X X^\top \preceq \wt{X} \wt{X}^\top \preceq (1+ 2 \alpha \beta/\eta) XX^\top
    $.
\end{lemma}
\begin{proof}
Let $i \in [d]$ be index that $X_{*,i}$ and $\wt{X}_{*,i}$ are different (See Definition~\ref{def:neighboring}).

We have
\begin{align*}
\wt{X} \wt{X}^\top 
= & ~ \sum_{j=1}^d \wt{X}_{*,j} \wt{X}_{*,j}^\top \\
= & ~ (\sum_{j \in [d] \backslash \{i\}} \wt{X}_{*,j} \wt{X}_{*,j}^\top ) +  \wt{X}_{*,i} \wt{X}_{*,i}^\top \\
= & ~ (\sum_{j \in [d] \backslash \{i\}} X_{*,j} X_{*,j}^\top ) +  \wt{X}_{*,i} \wt{X}_{*,i}^\top \\
= & ~ XX^\top - X_{*,i} X_{*,i}^\top + \wt{X}_{*,i} \wt{X}_{*,i}
\end{align*}
where the first step is the result of matrix multiplication, 
the second step is from  simple algebra,
the third step follows from Definition~\ref{def:neighboring}, and the last step comes from simple algebra.

We know that
\begin{align}\label{eq:xx_top}
\| X_{*,i} X_{*,i}^\top -\wt{X}_{*,i} \wt{X}_{*,i} \| 
= & ~ \| X_{*,i} X_{*,i}^\top -  X_{*,i} \wt{X}_{*,i}^\top +  X_{*,i} \wt{X}_{*,i}^\top  -\wt{X}_{*,i} \wt{X}_{*,i} \| \notag\\
\leq & ~ \| X_{*,i} X_{*,i}^\top -  X_{*,i} \wt{X}_{*,i}^\top \| +  \|  X_{*,i} \wt{X}_{*,i}^\top  -\wt{X}_{*,i} \wt{X}_{*,i} \| \notag\\
\leq & ~ \| X_{*,i} \|_2 \cdot \| X_{*,i} - \wt{X}_{*,i} \|_2 + \| X_{*,i} - \wt{X}_{*,i} \|_2 \cdot \| \wt{X}_{*,i} \|_2 \notag \\
\leq & ~ 2 \alpha \beta
\end{align}
where the first step is from adding a new term $ X_{*,i} \wt{X}_{*,i}^\top$,
the second step follows from the triangle inequality,
the third step follows from Fact~\ref{fac:norm},
and the last step is due to Definition~\ref{def:dataset} and Definition~\ref{def:neighboring}.

Thus, we have
$\wt{X} \wt{X}^\top  
\succeq  ~ XX^\top - 2 \alpha \beta I_n 
\succeq  ~ (1-2\alpha\beta / \eta) XX^\top$.
Similarly, we have
$\wt{X} \wt{X}^\top  
\preceq  ~ XX^\top + 2 \alpha \beta I_n 
\preceq ~ (1+2\alpha\beta/\eta) XX^\top$.
\end{proof}

The following is the presentation of the additional sensitivity lemma, which further extends the conclusion of Lemma~\ref{lem:sensitivity} in Section~\ref{sec:sensitivity}. We use the following lemma in the proof of our main result, Theorem~\ref{thm:formal}, presented in Section~\ref{sec:main_result}.
\begin{lemma}\label{lem:sensitivity_from_spectral_to_F}
If the following conditions hold
\begin{itemize}
\item Let $\alpha$ and $\eta$ be defined in Definition~\ref{def:dataset}.
\item Let $\beta$ be defined in Definition~\ref{def:neighboring}. 
\item  $X$ and $\wt{X}$ are neighboring datasets such that
\begin{align*}
  (1-2 \alpha \beta/\eta)  X X^\top \preceq \wt{X} \wt{X}^\top \preceq (1+2 \alpha \beta/\eta) XX^\top
\end{align*}
\end{itemize}
Then, we have 
\begin{itemize}
\item Part 1.
\begin{align*}
  \|  (XX^\top)^{-1/2}  \wt{X} \wt{X}^\top (XX^\top)^{-1/2} - I \| \leq 2 \alpha \beta/\eta
\end{align*}
\item Part 2.
\begin{align*}
     \|  (XX^\top)^{-1/2}  \wt{X} \wt{X}^\top (XX^\top)^{-1/2} - I \|_F \leq 2 \sqrt{n} \alpha \beta/\eta
\end{align*}
\end{itemize}
\end{lemma}
\begin{proof}
The proof is straightforward, and we omit the details here.
\end{proof}




\end{document}